\documentclass[10pt,twocolumn,letterpaper]{article}

\usepackage{subcaption}
\usepackage{iccv}
\usepackage{times}
\usepackage{epsfig}
\usepackage{graphicx}
\usepackage{amsmath}
\usepackage{amsthm}
\usepackage{amssymb}
 
\usepackage{algorithm2e}
\usepackage{amsmath,amssymb,amsthm}
\usepackage{multirow}
\usepackage{float}
\usepackage[percent]{overpic}
\usepackage{bbm}
\usepackage{hhline}

\usepackage[pagebackref=true,breaklinks=true,colorlinks,bookmarks=false]{hyperref}

\iccvfinalcopy % *** Uncomment this line for the final submission

\def\iccvPaperID{2712} % *** Enter the CVPR Paper ID here

\ificcvfinal\fi

\newtheorem{thm}{Theorem}%[section]
\newtheorem{defi}[thm]{Definition}
\newtheorem{lem}[thm]{Lemma}

\newtheorem{cor}[thm]{Corollary}

\long\def\ignore#1{}

\def\calA{\mathcal{A}}

\def\calL{\mathcal{L}}

\def\calN{\mathcal{N}}
\def\calO{\mathcal{O}}

\newcommand{\setsub}[2]{{#1 \backslash #2}}
\newcommand{\setsubVS}{\setsub{V}{S}}

\newcommand{\energychange}[3]{\Delta E(#2 \leftarrow #1 \mid #3)}
\newcommand{\energychangedefS}{\energychange{x_S}{y_S}{z_\setsubVS}}

\newcommand{\energychangedefNS}{\energychange{x_S}{y_S}{z_{\calN(S)}}}

\def\DEE{\textsc{DEE}}

\def\Alpha{\textsc{$\alpha$-Exp}}
\def\HRP{\textsc{PR}}

\def\ProbDEE{Ours}
\begin{document}

%%%%%%%%% TITLE
\title{A discriminative view of MRF pre-processing algorithms}

\author{Chen Wang$^{1, 2}$ \qquad Charles Herrmann$^2$ \qquad Ramin Zabih$^{1, 2}$ \\ $^1$ Google Research \qquad $^2$ Cornell University \\
{\tt\small \{chenwang, cih, rdz\}@cs.cornell.edu}}

\maketitle
\thispagestyle{empty}

%%%%%%%%% ABSTRACT
\begin{abstract}
While Markov Random Fields (MRFs) are widely used in computer vision, they present
a quite challenging inference problem. MRF inference can be accelerated
by pre-processing techniques like Dead End Elimination (DEE)
\cite{Desmet:Nature92} or QPBO-based approaches \cite{kohli2008partial,kovtun2003partial,kovtun2004image} which compute the
optimal labeling of a subset of variables. These techniques are guaranteed to
never wrongly label a variable but they often leave a large number of
variables unlabeled. We address this shortcoming by interpreting
pre-processing as a classification problem, which allows us to trade off false
positives (i.e., giving a variable an incorrect label) versus false negatives
(i.e., failing to label a variable). We describe an efficient discriminative
rule that finds optimal solutions for a subset of variables. Our
technique provides both per-instance and worst-case guarantees concerning the
quality of the solution. Empirical studies were conducted over several
benchmark datasets. We obtain a speedup factor of 2 to 12 over expansion
moves \cite{BVZ:PAMI01} without preprocessing, and on difficult non-submodular
energy functions produce slightly lower energy.
\end{abstract}

\section{Pre-processing for MRF inference}
\label{sec:intro}

We address the inference problem for pairwise {Markov Random Fields}
(MRFs) defined over $n$ variables $x = (x_1, \ldots, x_n)$, where each
$x_i$ is labeled from a discrete label set $\mathcal{L}$. The MRF can be
viewed as a graph $G=(V, E)$ with a \textit{neighborhood system} $\calN: V \mapsto 2^V$.
To compute the MAP estimate we minimize the energy
\begin{equation}
E(x) = \sum_{i \in V} \theta_i(x_i) + \sum_{(i,j) \in E}\theta_{ij}(x_i, x_j)
\end{equation}
where $\theta_i$ and
$\theta_{ij}$ are \textit{unary} terms and \textit{pairwise} terms.
MRFs are widely used in applications such as image segmentation, stereo,
etc~\cite{OpenGM:IJCV15,SZSVKATR:PAMI08}. Unfortunately the MRF inference
problem is NP-hard even when $|\calL|=2$ (i.e. binary
labels)~\cite{KZ:PAMI04}.

Many algorithms involve some kind of pre-processing phase that seeks to
determine the value of a subset of variables, thus reducing the complexity of
the remaining combinatorial search problem. Pre-processing methods are
commonly used in conjunction with graph cuts, a technique that achieves strong
performance on both binary and multilabel MRF inference
\cite{SZSVKATR:PAMI08}. Graph cuts handle binary MRFs by reduction to
min-cut, which is then solved via max-flow (see \cite{Boros:DAM02,FZ:PAMI11}
for reviews). The most widely used graph cut methods for multi-label MRFs 
are move-making techniques, which generate a new proposal at each
iteration and reduce the multi-label problem into a series of binary
subproblems (should each variable stick with the old label or switch to the
proposed label) and then solved by max-flow/min-cut. Popular
algorithms in this family include expansion moves~\cite{BVZ:PAMI01} and their
generalization to fusion moves~\cite{Lempitsky:TPAMI10}.

The best known pre-processing 
methods are Dead End Elimination (DEE) \cite{Desmet:Nature92} and QPBO
\cite{Boros:DAM02,KR:PAMI07}, but there are a number of others
\cite{kohli2008partial,kovtun2003partial,kovtun2004image,shekhovtsov2014maximum,shekhovtsov2015maximum,swoboda2014partial,WZ:CVPR16}.
(Similar approaches are used for other NP-hard problems, a
prominent example is Davis-Putnam's pure literal rule for SAT
\cite{DavisPutnam:JACM60}.)

The key weakness of such methods is that they are inherently conservative,
since they only label variables whose value can be determined in every global
minimum. Yet the MRFs that occur in computer vision are so large that in
practice we almost never compute the actual global minimum.\footnote{See
 \cite{Ishikawa:PAMI03,Weiss:ICCV05} for rare counterexamples.} As a result,
a pre-processing step that is carefully designed to never prune the global
minimum is followed by a search step that almost never finds the global
minimum. Our fundamental observation is that the pre-processing step can be
viewed as a classification problem, and that existing pre-processing methods
are designed to avoid false positives (i.e., to never label a variable
incorrectly), at the cost of many false negatives (i.e., variables that are
left unlabeled). By revisiting this tradeoff we can design techniques where
the combination of the pre-processing step and the search step leads to better
overall performance, especially on the most difficult problems.

As an example, consider a tiny 8-connected binary MRF with 9 variables
(pixels), and suppose we wish to determine by pre-processing that the center
pixel should be labeled with 0. In order to soundly compute this by DEE or
QPBO, we need to establish that switching the center pixel from 1 to 0 will
always decrease the energy, no matter what the configuration of the
surrounding pixels. Yet as demonstrated in Table~\ref{tab:mini-MRF}, there are
local configurations that are quite unlikely.

\begin{table}
  \centering
  \begin{tabular}{|c|c|c|}
\hline
  1 & 0 & 1    
\\\hline
  0 & ? & 0    
\\\hline
  1 & 0 & 1    
\\\hline
  \end{tabular}
\hspace*{.2in}
  \begin{tabular}{|c|c|c|}
\hline
  1 & 1 & 1    
\\\hline
  1 & ? & 1    
\\\hline
  1 & 1 & 1    
\\\hline
  \end{tabular}
%% \hspace*{.2in}
%%   \begin{tabular}{|c|c|c|}
%% \hline
%%   0 & 0 & 0    
%% \\\hline
%%   0 & ? & 0    
%% \\\hline
%%   0 & 0 & 0    
%% \\\hline
%%   \end{tabular}
\vspace*{.1in}
  \caption{Minimal MRF example. The left configuration is an unlikely configuration for the
    neighborhood around the center pixel in the global minimum, compared to
    the configuration at right. Existing pre-processing methods treat all
    configurations equally, and as a result fail to label many variables.}
  \label{tab:mini-MRF}
\end{table}

\subsection{Outline and contributions}

We begin with a summary of related work, with an emphasis on DEE, QPBO and
QPBO-based pre-processing techniques. In Section~\ref{sec:discrim} we give our
discriminative criterion for pre-processing, motivated by examples like
Table~\ref{tab:mini-MRF}, and provide efficient approximations for the key
subproblems. The theoretical performance of our method is analyzed in
Section~\ref{sec:bounds}, and experimental results are given in
Section~\ref{sec:experiments}. Most proofs are deferred to the supplemental
material, which also contains additional experimental results.

\section{Related work}

%\subsection{Optimal partial labelings}

A popular approach to the inference problem is to find the optimal
labeling for a subset of the variables
\cite{Desmet:Nature92,kahl2012generalized,kohli2008partial,kolmogorov2010generalized,Savchynskyy13:combiLP,shekhovtsov2014maximum,shekhovtsov2015maximum,swoboda2014partial,windheuser2012generalized}. 
A partial labeling that holds in every global minimizer is said to be
\textit{persistent}~\cite{Boros:DAM02}. 
Techniques like QPBO \cite{Boros:DAM02,KR:PAMI07} find an optimal partial
labeling by enforcing an even stronger condition: a partial labeling
that will decrease the energy if it is substituted into any complete
labeling.\footnote{QPBO is naturally viewed as a pre-processing method since
it finds persistent partial labelings, and leaves the task of labeling the
remaining variables to some other algorithm.}
This stronger property is sometimes called an \textit{autarky}~\cite{Boros:DAM02},
which was generalized by \cite{shekhovtsov2014maximum}. QPBO in particular is
widely used in computer vision since it often finds the correct label for the
majority of the variables.

To make these notions precise, we introduce the following notation.
A \textit{partial labeling} $x_S$ is represented
by the subvector of $x$ indexed by $S \subseteq V$. Let $\mathcal{L}_S =
\Pi_{i \in S} \mathcal{L}$ be the label space of $x_S$. Given two partial
labelings $x_A$ and $x_B$ where $A \cap B = \emptyset$, we define $x_A
\oplus x_B$ to be the composition of $x_A$ and $x_B$.\footnote{Let $y = x_A
 \oplus x_B$ when $A \cap B = \emptyset$, then we have $y_i = (x_A)_i$ when
 $i \in A$ and $y_i = (x_B)_i$ when $i \in B$.} As an important special case,
we can write substituting a partial labeling $x_S$ into a full labeling $z$ as
$x_S \oplus z_{V \backslash S}$. 

Following~\cite{Boros:DAM02}, we can define \textit{persistency} and
\textit{autarky}:
\begin{defi}\label{defi:persistency}
A partial labeling $x_S$ is persistent if 
\begin{equation}\label{eq:persistency}
 x_S = x^*_S, \quad \forall x^* \in {\arg\min}_x E(x).
\end{equation}
\end{defi}
\begin{defi}\label{defi:autarky}
A partial labeling $x_S$ is an autarky if
  \begin{multline}\label{eq:autarky}
E(x_S \oplus z_{V \backslash S}) < E(y_S \oplus z_{V \backslash S}),
\\\textit{$\forall z_{V\backslash S} \in \mathcal{L}_{V \backslash S}$ and $\forall y_S \in \mathcal{L}_S$ where $y_S \ne x_S$}.
  \end{multline}
\end{defi}
Persistency is the key property for pre-processing, since it determines the
optimal value of a subset of the variables and thus reduces the remaining
combinatorial search problem. In general, though, checking for persistency is
intractable \cite{Boros:DAM02}.  All existing persistency algorithms appear
to check the autarky
property as a sufficient condition, which states that overwriting an arbitrary
labeling with this partial labeling will reduce the energy.

\subsection{MRF pre-processing algorithms}

%%% TODO: Most of this is from CVPR16 paper, re-write slightly

QPBO generalizes the binary graph cut reduction that uses max-flow 
to find an optimal partial labeling
\cite{Boros:DAM02,KR:PAMI07,Rother:CVPR07}. If the energy function is 
submodular\footnote{For every pairwise cost, we have
  $\theta_{ij}(0,0)+\theta_{ij}(1,1) \le \theta_{ij}(0,1) +
  \theta_{ij}(1,0)$.} the partial labeling is complete (i.e., it labels
every variable and finds a global minimizer). However, the computational expense of
running max-flow is non-trivial. 

There are also techniques directly finding optimal partial labeling for the multi-label case, 
but the computational costs for these methods are
significant. Kovtun~\cite{kovtun2003partial,kovtun2004image} described an
approach constructing a series of binary one-verse-the-rest 
auxiliary problems and solve each of them via graph cuts.
MQPBO~\cite{kohli2008partial} and generalized roof
duality~\cite{windheuser2012generalized} proposed generalizations of QPBO to
multi-label MRFs.

Recently, Swoboda~et.~al.~\cite{swoboda2014partial} use standard MRF inference
algorithms to iteratively update the set of persistent
variables. Shekhovtsov~\cite{shekhovtsov2014maximum} formalized the problem to
maximize the number of optimally labeled variables as an LP. They also
proposed to combine these two approaches together which can take advantage of
both of them~\cite{shekhovtsov2015maximum}. The number of variables labeled by
these approaches are significantly more than Kovtun's approach and
MQPBO. However, the running time of these approaches is significantly longer,
since these approaches involve solving complex programming (either via
standard MRF inference solver or LP solver) iteratively.

Dead End Elimination (DEE)~\cite{Desmet:Nature92} and the recent Persistency
Relaxation (PR) algorithm \cite{WZ:CVPR16} are the only existing method
with cheaper computational costs than max-flow. DEE checks a local sufficient
condition which only involves a single vertex and its adjacent edges. PR
generalizes DEE to check a larger partial labeling, which gives improved
results on standard benchmarks.

Methods that optimally label a subset of the variables can obviously be used
to pre-process and accelerate MRF inference algorithms such as expansion moves. For example,
Radhakrishnan and Su~\cite{radhakrishnan2006dead} used DEE while
Alahari~et.~al.~\cite{alahari2010dynamic} applied Kovtun's approach.

\section{Discriminative pre-processing of MRFs}\label{sec:discrim}

In computer vision, the MRF inference problem is almost never solved
exactly. As a result, pre-processing methods that enforce soundness are far
too conservative, since they leave a large number of variables unlabeled. If we
view pre-processing as a binary classification problem (given a partial labeling
$x_S$, decide if it's persistent),
existing techniques ensure
that there are no false positives (i.e., variables given a label must be part of every global minimum), but at the cost of multiple false
negatives (i.e., variables that are left unlabeled).

\def\NS{{\calN(S)}}
\def\Ni{{\calN(i)}}
\def\energychangedefi{\energychange{x_i}{y_i}{z_\Ni}}
First, we need some notation. 
Define 
\[
\energychangedefS = E(y_S \oplus z_\setsubVS) - E(x_S \oplus
z_\setsubVS)
\]
to be the energy change when we substitute $x_S$ by $y_S$ given the partial
labeling $z_\setsubVS$ for the variables not in $S$. By expanding the
definition of $E(x)$ and cancelling terms, the Markov property of MRFs gives
us a sum over terms only depending on $x_i$, $y_i$ for $i \in S$ and $z_j$ for
$j \in V \backslash S$ with some $i \in S$ such that $(i,j) \in E$ (i.e.,
$z_j$ is adjacent to $S$). Let $\NS = \{j \in \setsubVS \mid \exists i \in
S, (i, j) \in E\}$, and we can rewrite $\energychangedefS =
\energychangedefNS$.

This allows us to
rewrite the autarky property~\eqref{eq:autarky} as: 
\begin{equation}\label{eq:autarky2}
\min_{y_S \ne x_S} \energychangedefNS > 0, \forall z_{\NS} \in \calL_{\NS}
\end{equation}

The key issue is the universal quantification in
Eq.~\ref{eq:autarky2}. To ensure that a partial labeling $x_S$
presents in all global minimizer, 
we look at all possible values that the
neighbors might have. For each of these, we check that any other assignment
$y_S$ would increase the energy. 

Yet this is obviously quite conservative. We now show the desired persistency
property can be rewritten by only looking at assignments to the neighboring
variables that occur in a global minimizer. 
Define $\calL^*_{\NS} = \{z^*_\NS \mid z^* \in
\arg\min E(z)\}$ be all possible configurations of $\NS$ in a global minimizer.
\begin{lem}
$x_S$ is persistent if and only if
\begin{equation}\label{eq:persistency2}
\min_{y_S \ne x_S} \energychangedefNS > 0, \forall
z_{\NS} \in \calL^*_{\NS}.
\end{equation}
\end{lem}
\begin{proof}
The if direction is trivial: consider an arbitrary global minimizer $z^*$, we have
$z^*_{\NS} \in \calL^*_{\NS}$ by definition. Suppose $x_S \ne z^*_S$, we will
have $E(x_S \oplus z^*_{V \backslash S}) < E(z^*)$, which contradicts the
assumption that
$z^*$ is
a minimizer. Therefore, we have $x_S = z^*_S, \forall z^*$, so it is
persistent.

For the only if direction, suppose Eq.~\ref{eq:persistency2} is not true, then
$\exists z_{\NS} \in \calL^*_{\NS}, \exists y_S \ne x_S$ such that
$\energychangedefNS \le 0$. We can expand $z_{\NS}$ to one minimizer $z^*$
such that $z^*_{\NS} = z_{\NS}$. Since $x_S$ is persistent, we also know
$z^*_S = x_S$. Therefore, $E(y_S \oplus z^*_{V \backslash S}) \le E(x_S
\oplus z^*_{V \backslash S}) = E(z^*)$. Since $z^*$ is a minimum this
inequality is an equality, hence $y_S \oplus z^*_{V \backslash S}$ is also a
global minimum. This contradicts the assumption that $x_S$ is persistent, since $y_S
\ne x_S$. 
\end{proof}

\subsection{Discriminative criterion}
Comparing Eq.~\ref{eq:autarky2} and Eq.~\ref{eq:persistency2}, we 
immediately observe that the universal quantifier makes autarky a
sound but stronger condition than persistency. Crucially, this suggests a
discriminative criterion to trade off false positives against false negatives. 

%% RDZ: I need to have this in the intro somehow. 
The high level idea is the following. Let $\hat \calL_{\NS}(x_S) = \{z_{\NS} \in
\calL_{\NS} \mid \min_{y_S \ne x_S} \energychangedefNS > 0\}$ be the set of
neighbor configurations $z_{\NS}$
such that given them $x_S$ is always a
better choice. When $\hat \calL_{\NS}(x_S)$ is large enough or covers the most
important neighbor configurations, it's very likely that we will
have $\calL^*_{\NS} \subseteq \hat \calL_{\NS}(x_S)$. This in turn implies $x_S$ is
persistent, even though $\hat \calL_{\NS}(x_S) \ne \calL_{\NS}$ and we do not
precisely know $\calL^*_{\NS}$.

Formally, assume we have a ground truth distribution $p(z_{\NS})$ which is
uniform over $\calL^*_{\NS}$ and $0$ otherwise. Then a sound condition to
check persistency is $\sum_{z_{\NS} \in \hat \calL_{\NS}(x_S)} p(z_{\NS}) = 1$. Of
course, computing $\calL^*_{\NS}$ and $p(z_{\NS})$ is computationally
intractable. So we use an estimated distribution $q(z_{\NS})$ that
approximates $p(z_{\NS})$. Looking back to Table~\ref{tab:mini-MRF}, one would
assume that the left configuration would not appear in $Z^*_{\NS}$, while the
right one quite plausibly could; there should be a lower $q$ value for the
left one but a higher $q$ value for the right.
Our discriminative criterion for persistency is
\begin{equation}\label{eq:decisionrule}
\sum_{z_{\NS} \in \hat \calL_{\NS}(x_S)} q(z_{\NS}) \ge \kappa.
\end{equation}
Here $\kappa \in [0, 1]$ is the key parameter that controls the tradeoff
between false positives and false negatives, as shown by the following
(obvious) lemma.
% Chen: do we actually need this trivial lemma?
\begin{lem}
For the same set of decision problems for persistency, we will never increase
the number of false positives by increasing $\kappa$.
\end{lem}

We now address the two crucial issues: how to choose $q$ to effectively
approximate $p$, and how to efficiently check Eq.~\ref{eq:decisionrule}.

\subsection{Approximating $p$}\label{sec:approxprob}
A trivial baseline is to treat each $z_\NS$ as equally important and set our
approximation $q(z_\NS)$ to be the uniform distribution over $\calL_\NS$. In
this special case, Eq.~\ref{eq:decisionrule} is equivalent to count
the number of neighbor configurations $z_\NS$ that satisfy $\min_{y_S \ne
 x_S} \energychangedefNS > 0$. We expect $\hat \calL_\NS(x_S)$ to cover the
unknown $\calL^*_\NS$ with high probability when $|\hat \calL_\NS(x_S)|$ is large
enough.

A more elegant approach is to estimate the marginal probability of a
particular assignment $z_\NS$ via the generative MRF model, and use this as
our approximation for $p$. This problem is well studied in the message
passing literature, and is often solved by max-product loopy belief
propagation (LBP) \cite{Murphy12,WF:TIT01}. 
An important special case is if we only use the initialization of LBP,
$q_i(z_i) \propto e^{-\theta_i(z_i)}$. This makes a certain amount of
intuitive sense: in the MRF energy functions that occur in computer
vision it is well known that most of the weight comes from the unary terms 
\cite{Murphy12},
which provide a strong signal as to the optimal label for each variable.

More generally, we can define $q(z_{\NS})$ to be a fully independent
distribution $q(z_{\NS}) = \Pi_{i \in \NS} q_i(z_i)$ with
$q_i(z_i) \propto e^{-\theta_i(z_i))}\Pi_{j \in \calN(i)} m_{j \rightarrow
 i}(z_i)$, where $m_{j \rightarrow i}(z_i)$ is the message we have from the
belief propagation algorithm. Since this is just an approximation, we would
not need to pay the cost of running LBP to convergence. In our experiments,
the more general approach does not seem to pay dividends, but other ways of
estimating the marginals are worth investigating.
%% RDZ: Nice way to say this doesn't work

\subsection{Efficiently checking our discriminative criterion}\label{sec:approxsum}
Checking Eq.~\ref{eq:decisionrule} is generally computational intractable,
due to the size of $\calL_\NS(x_S)$ and $\{y_S \in \calL_S \mid y_S \ne x_S\}$.
We now propose a polynomial time algorithm to compute a lower bound for
$\sum_{z_{\NS} \in \hat \calL_{\NS}(x_S)} q(z_{\NS})$. 

We will focus on the persistency of a single variable $x_i$ from this point
forward. This subroutine is used by our construction algorithm (which will be
described in Section~\ref{sec:algorithm}) to construct a partial
labeling for the given energy function $E(x)$. However, our methods can handle
an arbitrary $x_S$ for $|S| > 1$; the details are deferred to the
supplementary material, but are similar to the single variable case.

Our general strategy is to find a subset of $\calL_\Ni$ which we know is inside $\hat \calL_\Ni(x_i)$ and can be easily factorized. We start by considering each node $j \in \Ni$ independently. For each $j$, define $\calA_j$ to be the set of labels $\ell$ where the autarky condition holds if $z_j = \ell$. Since autarky is a stronger condition than persistency, we know that all $z_\Ni$ values where $z_j\in \calA_j$ are inside $\hat \calL_\Ni(x_i)$. The union of these sets across different $j\in \Ni$ will still be a subset of $\hat \calL_\Ni(x_i)$.

Formally, define $\calL_\Ni^{z_j = \ell} = \{z_\Ni \mid z_j = \ell\}$. Then $\calA_j = \{ \ell \mid \min_{y_i \ne x_i} \energychangedefi > 0, \forall z_\Ni \in \calL_\Ni^{z_j = \ell}\}$. Let $\calL_\Ni^{z_j\in\calA_j} = \cup_{\ell \in \calA_j} \calL_\Ni^{z_j = \ell}$. Then, we know that $\calL_\Ni^{z_j \in \calA_j} \subseteq \hat \calL_\Ni(x_i)$ and $\cup_{j \in \Ni}\calL_\Ni^{z_j \in \calA_j} \subseteq \hat \calL_\Ni(x_i)$.

We establish a computationally tractable lower bound for $\sum_{z_{\Ni} \in
 \hat \calL_{\Ni}(x_i)} q(z_{\Ni})$ by the following lemma, which we can check 
instead. 

\begin{lem}\label{lem:approxsum}
We have the following lower bound:
\begin{equation}
\sum_{j \in \Ni} Q_j \Pi_{k \in \Ni, k \prec j}(1 - Q_k) \le \sum_{z_{\Ni} \in \hat \calL_{\Ni}(x_i)} q(z_{\Ni}),
\end{equation}
where $Q_i = \sum_{\ell \in \calA_i} q_i(z_i = \ell)$.
\end{lem}
\begin{proof}

We can view $\sum_{z_{\Ni} \in \calL'_\Ni} q(z_{\Ni})$ as the probability
$Pr(z_{\Ni} \in \calL'_\Ni)$ given distribution $q$. 

Because our $q(z_\Ni)$ can be factorized independently, we can
integrate over the variables other than $z_j$ to get 
$Pr(z_{\Ni} \in
\calL_\Ni^{z_j = \ell}) = Pr(z_j = \ell) = q_j(z_j = \ell)$. 

We also have $Pr(z_{\Ni} \in \calL_\Ni^{z_j\in \calA_j}) = Pr(z_j \in
\calA_j) = \sum_{\ell \in \calA_j} q_j(z_j = \ell) = Q_j$ since
$\calL_\Ni^{z_j = \ell}$ are all disjoint. Then, using independence again, we have
\begin{equation}\label{eq:fastcheck}
\begin{aligned}
 & \sum_{z_{\Ni} \in \cup_{j \in \Ni} \calL_\Ni^{z_j\in \calA_j}} q(z_{\Ni}) \\
= & Pr\Big(z_{\Ni} \in \cup_{j \in \Ni} \calL_\Ni^{z_j\in \calA_j}\Big) \\
= & Pr\Big(\cup_{j \in \Ni} \big(z_{\Ni} \in \calL_\Ni^{z_j\in \calA_j}\big)\Big) \\
= & Pr\Big(\cup_{j \in \Ni} \big( z_j \in \calA_j \big)\Big) \\
= & Pr(z_{j_1} \in \calA_{j_1}) + Pr(z_{j_2} \in \calA_{j_2})Pr(z_{j_1} \not\in \calA_{j_1}) \cdots\\
= & \sum_{j \in \Ni} Q_j \Pi_{k \in \Ni, k \prec j}(1 - Q_k) % \vspace*{-.5in}
\end{aligned}
\end{equation}
Finally, note that we argued $\cup_{j \in \Ni} \calL_\Ni^{z_j\in \calA_j} \subseteq
\hat \calL_\Ni(x_i)$ before, which concludes the proof. 
\end{proof}

Constructing $\calA_j$ requires us to be able to efficiently check 
$\min_{y_i \ne x_i} \energychangedefi > 0, \forall z_\Ni \in
\calL_\Ni^{z_j = \ell}$. We 
expand it by the definition of $E(x)$ then swap the min and sum
operators. This gives the following lower bound, which we check for being 
strictly positive: 
\begin{equation}\label{eq:fastDEE}
\begin{aligned}
& \min_{y_i \ne x_i} \big(\theta_i(y_i) - \theta_i(x_i)\big) + \min_{y_i \ne x_i}\big(\theta_{ij}(y_i, \ell) - \theta_{ij}(x_i, \ell)\big)\\ 
& + \sum_{k \in \Ni, k \ne j} \min_{z_k, y_i \ne x_i} \big(\theta_{ij}(y_i, z_k) - \theta_{ij}(x_i, z_k)\big) > 0
\end{aligned}
\end{equation}

\subsection{Our algorithm}\label{sec:algorithm}

\RestyleAlgo{boxruled}
\LinesNumbered
\begin{algorithm}\label{alg:construction} % [H]
 \SetAlgoNoLine
 \KwIn{Energy function $E(x)$}
 $\hat x \gets \emptyset$; \quad $S \gets \emptyset$\;
 \For{$t \gets 1$ \textbf{to} $\tau$}{
 	\For{$i \in V \backslash S, \ell \in \calL_i$}{
		%\If{$|\calL_i| = 1$}{\textbf{continue}\;}
		Compute $LB \le\sum_{z_\Ni \in \hat \calL_\Ni(x_i = \ell)} q(z_\Ni)$\;
		\If{$LB \ge \kappa$}{
			$\hat x \gets \hat x \oplus \{x_i = \ell\}$\;
			$\calL_i \gets \{\ell\}$; \quad $S \gets S \cup \{i\}$\;
		}
 	}
 }
 With $\hat x_S$ fixed, use one MRF inference algorithm to solve the remaining variables, get $\hat x_{V \backslash S}$\;
 \Return $\hat x = \hat x_S \oplus \hat x_{V \backslash S}$\;
 \caption{MRF inference with pre-processing}
\end{algorithm}

We have presented our discriminative criterion to decide if a given partial
labeling $x_i = \ell$ is persistent. Now we will use it as a key
subroutine to perform pre-processing for MRF inference, as shown in
lines 2-10 of Algorithm~\ref{alg:construction}. 
We firstly loop over the unlabeled variables and its label set (line 3). 
For each given $x_i = \ell$, use our discriminative rule to judge if it's
persistent (line 4-8). We will fix its value if it passes
our test by setting $\calL_i = \{\ell\}$, and concatenate it with our
inference result $\hat x$ (line 6, 7). Note that fixing $x_i = \ell$ will
also provide additional information as to the unlabeled variables which were
checked before $x_i$, so we repeat the whole procedure for $\tau$
iterations (line 2). 

%An interesting detail is that computing $q$ using LBP is time consuming, yet
%we experimentally observe that our algorithm is robust against an imperfect
%distribution $q$. So we simply estimate $q$ globally by LBP and never update
%it even when we fix the values of some variables. 

After our pre-processing has
terminated and labeled the variables in the set $S$, we fix the variables
$\hat x_S$ and use any MRF inference algorithms to solve the remaining energy
minimization problem, which gives us a labeling $\hat x_{V \backslash S}$ on
the remaining variables (line 11). Finally, we obtain our inference result by
concatenating them together (line 12).

\noindent\textbf{Running time} We now give an asymptotic bound on the running time of
our pre-processing algorithm here, deferring the analysis into the
supplementary material. Assuming we have an oracle to give us data terms
$\theta_i(x_i)$ and prior terms $\theta_{ij}(x_i, x_j)$ in $\calO(1)$
time. Let $N = |V|, M = |E|$ and $L = \max_{i} |\calL_i|$ be the number of
variables, edges and maximum possible labels, and $d = \max_i |\Ni|$ be the
maximum degree of the graph. The overall running time is 
$\calO(d^2NL^2 +EL^2)$ when we use Section~\ref{sec:approxsum} to check our
discriminative criterion, and $\calO(dNL^{d+2})$ for brute force (which is
feasible when both $d$ and $L$ are small constants).

\section{Performance bounds}\label{sec:bounds}

We can analyze the per-instance and worst-case performance of our
pre-processing methods when followed by an inference algorithm that produces a
solution with performance bounds.

\subsection{Per-instance bounds}\label{sec:perinstancebound}

There are a number of MRF inference algorithms that produce per-instance
guarantees (i.e., they produce a certificate after execution that their
solution is close to the global minimum). These methods, which are typically
based on linear programming, include
\cite{Kolmogorov:TRWS06,Komodakis:PAMI07,Wainwright:TRW05}, and they
provide a per-instance 
additive error bound by computing the duality gap. 
% Chen, I think this is an old sentence?
%% The theoretical
%% results given above can be naturally generalized to the situation where our
%% pre-processing methods are followed by such inference techniques.

Our algorithm has a natural way to bound additive errors. Recall our notation
$\energychangedefi$ describing the energy changes when we flip $x_i$ to $y_i$ with
the neighbor configuration $z_\Ni$. Therefore, 
$\min_{z_\Ni} \min_{y_i} \energychangedefi \le 0$ 
is the worst case energy decrement when we flip $x_i$
to arbitrary $y_i$ with arbitrary neighbor configurations $z_\Ni$. It's
non-positive since we can always set $y_i = x_i$. Now we can negate it and
define $\delta_i \triangleq -\min_{z_\Ni} \min_{y_i} \energychangedefi$ to be
the maximum potential energy loss when we use our discriminative criterion to
decide $x_i$ is persistent. Then we have the following two lemmas. 

\begin{lem}\label{lem:additivelemma}
Let $\hat x_S$ be the persistent variables found by our
Algorithm~\ref{alg:construction}. For arbitrary $\hat x_{V \backslash S}$,
and arbitrary $x'_S$, we have 
$E(\hat x_S \oplus \hat x_{V 
  \backslash S}) \le E(x'_S \oplus \hat x_{V \backslash S}) + \sum_{i \in S}
\delta_i$. 
\end{lem}
\begin{proof}
With $\hat x_{V \backslash S}$ fixed, we flip $\hat x_i$ to $x'_i$ in the
reverse order of them being added to $S$ by our algorithm. Due to the
analysis before, we will lose at most $\delta_i$ at each step. 
\end{proof}

\begin{thm}\label{lem:peradditivebound}
Suppose the inference algorithm has per-instance $\zeta$-additive bound,
then $E(\hat x) \le E(x^*) + \zeta + \sum_{i \in S} \delta_i$. 
\end{thm}
\begin{proof}
Let $\bar x_{V \backslash S}$ be the minimizer of $E(x)$ with $\hat x_S$
fixed, which might be different than the global minimizer $x^*_{V \backslash
 S}$. Then we will have 
$E(\hat x_S \oplus \hat x_{V \backslash S}) \le 
E(\hat x_S \oplus \bar x_{V \backslash S}) + \zeta \le E(\hat x_S \oplus
x^*_{V \backslash S}) + \zeta \le E(x^*_S \oplus x^*_{V \backslash S}) + \zeta
+ \sum_{i \in S}\delta_i$. 
The first step is because we use an inference
algorithm with $\zeta$-additive errors to solve the problem with $\hat x_S$
fixed. The second step follows because $\bar x_{V \backslash S}$ is the minimizer
w.r.t. $\hat x_S$.
\end{proof}

As a special case, any sound condition like Eq.~\ref{eq:autarky2}
guarantees $\delta_i = 0$, i.e., we don't make mistakes. In practice it is
computationally intractable to compute $\delta_i$, so just as in
Section~\ref{sec:approxsum} we swap the min and sum operators, and compute the
upper bound $\bar \delta_i \ge \delta_i$ efficiently. Then 
we use $\sum_{i \in S} \bar \delta_i$ as our per-instance additive bound.

\subsection{Worst case bounds}\label{sec:worstcasebound}

Some MRF inference algorithms produce a solution that is guaranteed to lie
within a known factor of the global minimum. The best known such technique is
the expansion move algorithm \cite{BVZ:PAMI01} but there are others
\cite{gould2009alphabet,KT:JACM02,Komodakis:PAMI07}. We can easily turn
our per-instance bounds into the worst case bounds. We combine
Eq.~\ref{eq:decisionrule} and $\bar \delta_i \le \epsilon$ as our
discriminative criterion for pre-defined $\epsilon$.

\begin{cor}\label{lem:worstadditivebound}
Suppose the inference algorithm has worst case $\zeta$-additive bound,
then $E(\hat x) \le E(x^*) + \zeta + |S|\epsilon$ is our worst
case additive bound. 
\end{cor}

Inference algorithm with worst case guarantees are usually multiplicative
bounds other than additive bounds, but we can modify our proof of
Theorem~\ref{lem:peradditivebound} to get the following bounds.

\begin{thm}\label{lem:worstmultbound}
Suppose the inference algorithm has a worst case $\beta$-multiplicative bound,
then we will have $E(\hat x) \le \beta\cdot E(x^*) + \beta\cdot |S|\epsilon$.
\end{thm}
\begin{proof}
Following the proof of Theorem~\ref{lem:peradditivebound}, we have $E(\hat x_S \oplus \hat x_{V \backslash S}) \le \beta\cdot E(\hat x_S \oplus \bar x_{V \backslash S}) \le \beta \cdot E(\hat x_S \oplus x^*_{V \backslash S}) \le \beta\cdot(E(x^*_S \oplus x^*_{V \backslash S}) + |S|\epsilon)$.
\end{proof}

A more careful analysis can give us a tighter bound (dropping the coefficient
$\beta$ before $|S|\epsilon$), for the important special case where we use the
expansion move algorithm \cite{BVZ:PAMI01} for inference. We defer the
proof to the supplementary material. 

\begin{thm}\label{lem:worstmultbound}
Suppose we use expansion moves as the inference algorithm, with the
$\beta$-multiplicative bound, then we will have $E(\hat x) \le \beta\cdot
E(x^*) + |S|\epsilon$. 
\end{thm}

\section{Experiments}\label{sec:experiments}
\subsection{Datasets and experimental setup}
\noindent
\textbf{Approaches} The most natural baselines for us to compare against
include inference without pre-processing, and inference using the sound (but
conservative) DEE~\cite{Desmet:Nature92} and PR~\cite{WZ:CVPR16} techniques.
We employ expansion moves for MRF inference~\cite{BVZ:PAMI01}. In order
to achieve better speedup, we apply preprocessing to each induced binary
subproblem of expansion moves as the input of DEE, PR or
Algorithm~\ref{alg:construction}, and then run QPBO \cite{KR:PAMI07}.
%% For the unlabeled
%% variables, we follow the standard convention in the graph cuts literature,
%% which is that they retain their previous label.

%% We also have some
%% preliminary results on applying our approach to multi-label energy directly. 
%% We report them in the supplementary material.

At the other end of the spectrum are high overhead
techniques such as Kovtun's approach~\cite{kovtun2003partial,kovtun2004image},
MQPBO~\cite{kohli2008partial}, and LP-based
approaches~\cite{shekhovtsov2014maximum,shekhovtsov2015maximum,swoboda2014partial}.
These algorithms require more running time than max-flow on each
induced binary subproblem. Therefore, we apply them to the multi-label problem, and then
use expansion move to infer the remaining part. We choose the IRI 
method~\cite{shekhovtsov2015maximum} as the representative among 
\cite{shekhovtsov2014maximum,shekhovtsov2015maximum,swoboda2014partial} since
it's significantly faster.
Note that the R$^3$~\cite{alahari2010dynamic} method
also uses Kovtun's method as their pre-processing (reduce) step in order to speed up MRF inference.
The reuse and recycle parts attempt to speed up the inference algorithm itself, which is orthogonal to
what we propose to do in this paper, so we do not compare against this method.

We also compared against other widely used MRF inference algorithms besides
expansion moves, including loopy belief propagation
(LBP)~\cite{Murphy12,WF:TIT01}, dual decomposition
(DD)~\cite{Kappes12:Bundle}, TRWS~\cite{Kolmogorov:TRWS06} and
MPLP~\cite{globerson2008fixing, sontag2012efficiently, sontag2008tightening}.
The comparison among these inference algorithms are provided in survey
papers~\cite{OpenGM:IJCV15,SZSVKATR:PAMI08}. In our experiments, expansion
moves is usually significantly faster than other methods, and gives
comparable or better energy. These experimental comparisons are
deferred to the supplemental material.
% RDZ: CW discussion goes here

\noindent\textbf{Dataset} 
We conducted experiments on a variety of MRF inference benchmarks, where the
energy minimization problems come from different vision problems, including
color segmentation~\cite{lellmann2011continuous}, 
stereo, image inpainting, denoising~\cite{SZSVKATR:PAMI08} and optical flow~\cite{chen2016full}. 
Datasets for the first three tasks are wrapped in OpenGM2~\cite{OpenGM:IJCV15} and are available online.
We use the BSDS300~\cite{Martin:FTM01} for the denoising task with the MRF setup following~\cite{SZSVKATR:PAMI08}.
We use the MPI Sintel dataset~\cite{Butler:ECCV:2012} for the optical flow task with the MRF setup following~\cite{chen2016full}.

Our focus, of course, is on the difficult inference problems where the induced
binary subproblem is non-submodular. For comparison, we also included some
experiments on relatively easy problems where the induced binary subproblem is
submodular.

\noindent\textbf{Measurement} 
We report the improvement in overall running time (including both
pre-processing and the inference for the remaining unlabeled variables) and
relative energy change.\footnote{Our goal is efficient
energy minimization, so speed and energy are the key criterion,
but we also provide visual results comparison in the supplemental material.}
The baseline is expansion moves with no pre-processing.
Let $T_i^{\text{ALG}}$ and $E_i^{\text{ALG}}$ be the
running time and energy for algorithm ALG on the $i$-th instance. We 
define the \textit{speedup} as $T_i^{\Alpha{}} / T_i^{\text{ALG}}$ and \textit{energy change}
$(E_i^{\text{ALG}} - E_i^{\Alpha{}}) / E_i^{\Alpha{}}$ for each instance, and then
report the average speedup and energy change for the whole dataset.

We also report the percentage of labeled variables during the pre-processing.
Since we view the decision problem (whether a given partial labeling is
persistent) as a classification problem, we interchangeably use the term
\textit{percentage of labeled variables} and \textit{recall} value. Getting the precision
value is tricky. Since it's a NP-hard problem so we cannot
have the ground truth label for every variable. However, we apply our
pre-processing technique to the binary subproblems induced from expansion
moves. We know that either max-flow solves the subproblem exactly for the submodular
cases or QPBO can find a sufficiently large subset of partial persistent
labeling for the non-submodular cases (in our experiments, it labels almost all the variables). Therefore,
we report the \textit{precision} value of our method on the subset of the
variables where we know the ground truth labeling.

%\ignore{
%Therefore, let $x^{\text{ALG}}$ be the persistent partial
%labeling our algorithm claims on $V^{\text{ALG}}$ and $x^{\text{OPT}}$ be the
%correct persistent partial labeling claimed by max-flow or QPBO on $V'$. Then
%we define precision as $|\{i \in V^{\text{ALG}} \cap V' | x^{\text{ALG}}_i =
%x^{\text{OPT}}_i\}| / |V^{\text{ALG}} \cap V'|$, recall as $|\{i \in V^{\text{ALG}} \cap V' | x^{\text{ALG}}_i =
%x^{\text{OPT}}_i\}|/ |V'|$ and F1 score being their harmonic mean.
%}
%
%\ignore{We evaluate the different approaches in two
%respects. First, we report the improvement in overall running time for the preprocessing and
%the inference on the remaining undetermined variables. We use \Alpha{} as the baseline
%and report the speedup for other methods compared to \Alpha{}.
%%which we break down into overall running time, preprocessing time and
%%flow computation time.\footnote{Here, the preprocessing time and flow
%%  computation time are summed over all the iterations, not per iteration time.}
%%Note that the time spent in overhead like loading and uncompressing the model
%%from input files, writing the log and output are excluded from the total
%%running time.  
%The reported numbers are averaged over all instances for each
%dataset. 
%Second, we report the size of the partial optimal labeling found by
%the various preprocessing methods. The reported
%numbers are first averaged for each iteration of \Alpha{} in each instance,
%then averaged over all instances in one dataset.}

\noindent\textbf{Parameter setup and sensitivity analysis}
The discriminative rule in our approach has a few parameters. 
In order to achieve a fair comparison, we employed
leave-one-out cross-validation (see, e.g. \cite{Murphy12}) to use all but one instances in
the same dataset as the validation set to choose the best parameter\footnote{Based on the 
criterion that we choose the fastest overall running time when the false positive rate is less than $1\%$.}
and test on the remaining instance. We explored all the combinations of 1) threshold 
$\kappa \in \{0.7, 0.8, 0.9\}$, 2) using a uniform distribution or the distribution derived from 
the unary term for our $q(x)$, 3) using Section~\ref{sec:approxsum} to
compute $LB$ on line 4 of Algorithm~\ref{alg:construction} or 
using brute force to compute $\sum_{z_\Ni \in \hat \calL_\Ni(x_i = \ell)} q(z_\Ni)$ exactly, 
4) number
of iterations $\tau \in \{1, 3, 5, 7, 9\}$. We run expansion moves until
convergence or after 5 iterations through 
of the whole label set. We set our worst case bound $\epsilon = \infty$ in Section~\ref{sec:benchmarks}, 
in order to investigate how good our discriminative rule is even without the worst case guarantee.
We further study the role of $\epsilon$ in Section~\ref{sec:epsilon} and supplementary material.

There is evidence that our approach achieves good performance over a wide
range of parameters. We observed that cross validation picked nearly
identical parameters for every instance in the same dataset. 
Using nearby parameters also produced good results.

We also experimented with the following fixed parameters, to avoid the expense
of cross-validation: $\kappa = 0.8,
\tau = 3$, using the uniform distribution for $q(x)$ and checking with
Section~\ref{sec:approxsum}. Note that this is a fairly conservative assumption,
since we use the exact same parameters for very different energy
functions, but still obtain good results. We acheive a 2x-12x speedup
on different datasets with the energy increasing $~0.1\%$ on the worst case. In
addition, we still get lower energy on 4 of the 5 challenging dataset. We
defer the details of our fixed parameter experiments to the supplementary material.

\begin{table*}[!t]
\centering
\caption{ Experimental Results (N/A: not applicable, TO: time out, MEM: out of memory)}
\label{tab:experiments}
\footnotesize
\begin{tabular}{|cc|c|c||c||c|c||c|c|c|}
\hline
&&Dataset & Measurement & Ours & DEE & PR & Kovtun & MQPBO & IRI \\
\hline
\multirow{15}{*}{\rotatebox[origin=c]{90}{\bf Challenging Datasets}} &\hspace{-1em}\multirow{15}{*}{\rotatebox[origin=c]{90}{(non- Potts energy, large $|\calL|$)}} &\bf Stereo & Speedup & \bf 1.78x & 1.06x & 1.13x & N/A & MEM & 0.51x \\
&& 12--20 labels & Energy Change & -0.06\% & 0.00\% & 0.00\% & N/A & MEM & \bf -0.15\% \\
&& Trunc. L1/L2 & Labeled Vars & 44.76\% & 10.07\% & 18.06\% & N/A & MEM & \bf 56.45\% \\
\cline{3-10}
&& \bf Inpainting & Speedup & \bf 3.40x & 1.28x & 1.32x & N/A & MEM & 0.12x \\
&& 256 labels & Energy Change & \bf -1.71\% & 0.00\% & 0.00\% & N/A & MEM & 0.00\% \\
&& Trunc. L2 & Labeled Vars & \bf 74.29\% & 21.05\% & 23.75\% & N/A & MEM & 0.36\% \\
\cline{3-10}
&& \bf Denoising-sq & Speedup & \bf 11.83x & 1.20x & 1.37x & N/A & MEM & 0.29x \\
&& 256 labels & Energy Change & \bf -0.02\% & 0.00\% & 0.00\% & N/A & MEM & 0.00\% \\
&& L2 & Labeled Vars & \bf 97.91\% & 16.54\% & 29.83\% & N/A & MEM & 0.39\% \\
\cline{3-10}
&& \bf Denoising-ts & Speedup & \bf 11.91x & 10.53x & 10.64x & N/A & MEM & 0.18x \\
&& 256 labels & Energy Change & 0.00\% & 0.00\% & 0.00\% & N/A & MEM & \bf -0.03\% \\
&& Trunc. L2 & Labeled Vars & \bf 98.32\% & 95.65\% & 97.69\% & N/A & MEM & 5.85\% \\
\cline{3-10}
&& \bf Optical Flow & Speedup & \bf 4.69x & 2.63 & 3.40x & N/A & MEM & TO \\
&& 225 labels & Energy Change & \bf -0.04\% & 0.00\% & 0.00\% & N/A & MEM & TO \\
&& L1 & Labeled Vars & \bf 77.25\% & 54.34\% & 65.51\% & N/A & MEM & TO \\
\hline
\multirow{6}{*}{\rotatebox[origin=c]{90}{\bf Easy Datasets}} & \hspace{-1em}\multirow{6}{*}{\rotatebox[origin=c]{90}{(Potts, small $|\calL|$)}} & \bf Color-seg-n4 & Speedup & \bf 7.02x & 4.55x & 6.34x & 2.43x & 0.37x & 3.67x\\
&&4--12 labels & Energy Change & 0.00\% & 0.00\% & 0.00\% & 0.00\% & 0.00\% & \bf -0.12\% \\
&& Potts & Labeled Vars & 85.74\% & 65.38\% & 77.50\% & 70.32\% & 17.27\% & \bf 98.44\% \\
\cline{3-10}
&& \bf Color-seg-n8 & Speedup & \bf 8.33x & 5.61x & 6.37x & 2.33x & 0.32x & 1.45x \\
&&4--12 labels & Energy Change & +0.04\% & 0.00\% & 0.00\% & 0.00\% & 0.00\% & \bf -0.10\%\\
&& Potts & Labeled Vars & 90.39\% & 71.62\% & 82.05\% & 70.05\% & 17.87\% & \bf 99.35\% \\
\hline
\end{tabular}
\end{table*}%

\begin{table*}[!t]
\centering
\caption{Precision of our method}\label{tab:precision}
\footnotesize
\begin{tabular}{|c|c|c|c|c|c|c|c|}
\hline
Dataset & \bf Stereo & \bf Inpainting & \bf Denoising-sq & \bf Denoising-ts & \bf Optical Flow & \bf Color-seg-n4 & \bf Color-seg-n8 \\
\hline
Precision & 99.74\% & 96.16\% & 99.95\% & 99.79\% & 99.88\% & 99.79\% & 99.77\% \\
\hline
\end{tabular}
\vspace{-1em}
\end{table*}%

\begin{table}[!t]
\centering
\caption{Precision/recall value v.s. $\kappa$ (P: Precision, R: Recall)}\label{tab:kappa}
\footnotesize
\begin{tabular}{|c|c|c|c|c|c|}
\hline
\multicolumn{2}{|c|}{$\kappa$} & 0.7 & 0.8 & 0.9 & 1.0\\
\hline
\multirow{2}{*}{\bf Stereo} & P & 90.40\% & 99.71\% & 99.41\% & 100.00\% \\
& R & 91.31\% & 56.77\% & 11.35\% & 9.26\% \\
\hline
\multirow{2}{*}{\bf Inpainting} & P & 95.11\% & 99.88\% & 99.96\% & 100.00\% \\
& R & 90.51\% & 47.06\% & 25.97\% & 21.93\% \\
\hline
\multirow{2}{*}{\bf Denoising-sq} & P & 99.66\% & 99.95\% & 99.95\% & 100.00\% \\
& R & 99.47\% & 97.52\% & 19.11\% & 15.15\% \\
\hline
\multirow{2}{*}{\bf Denoising-ts} & P & 99.75\% & 99.95\% & 99.99\% & 100.00\% \\
& R & 98.61\% & 96.65\% & 94.99\% & 94.62\% \\
\hline
\multirow{2}{*}{\bf Optical Flow} & P & 94.01\% & 99.50\% & 99.98\% & 100.00\% \\
& R & 99.27\% & 93.74\% & 60.85\% & 56.79\% \\
\hline
\multirow{2}{*}{\bf Color-seg-n4} & P & 94.77\% & 99.50\% & 99.86\% & 100.00\% \\
& R & 98.52\% & 90.80\% & 77.20\% & 66.65\% \\
\hline
\multirow{2}{*}{\bf Color-seg-n8} & P & 99.48\% & 99.76\% & 99.87\% & 100.00\% \\
& R & 92.84\% & 90.43\% & 86.92\% & 71.66\% \\
\hline
\end{tabular}
\vspace{-1em}
\end{table}%

\begin{figure}[t]
% ,grid,tics=10
\includegraphics[width=\linewidth]{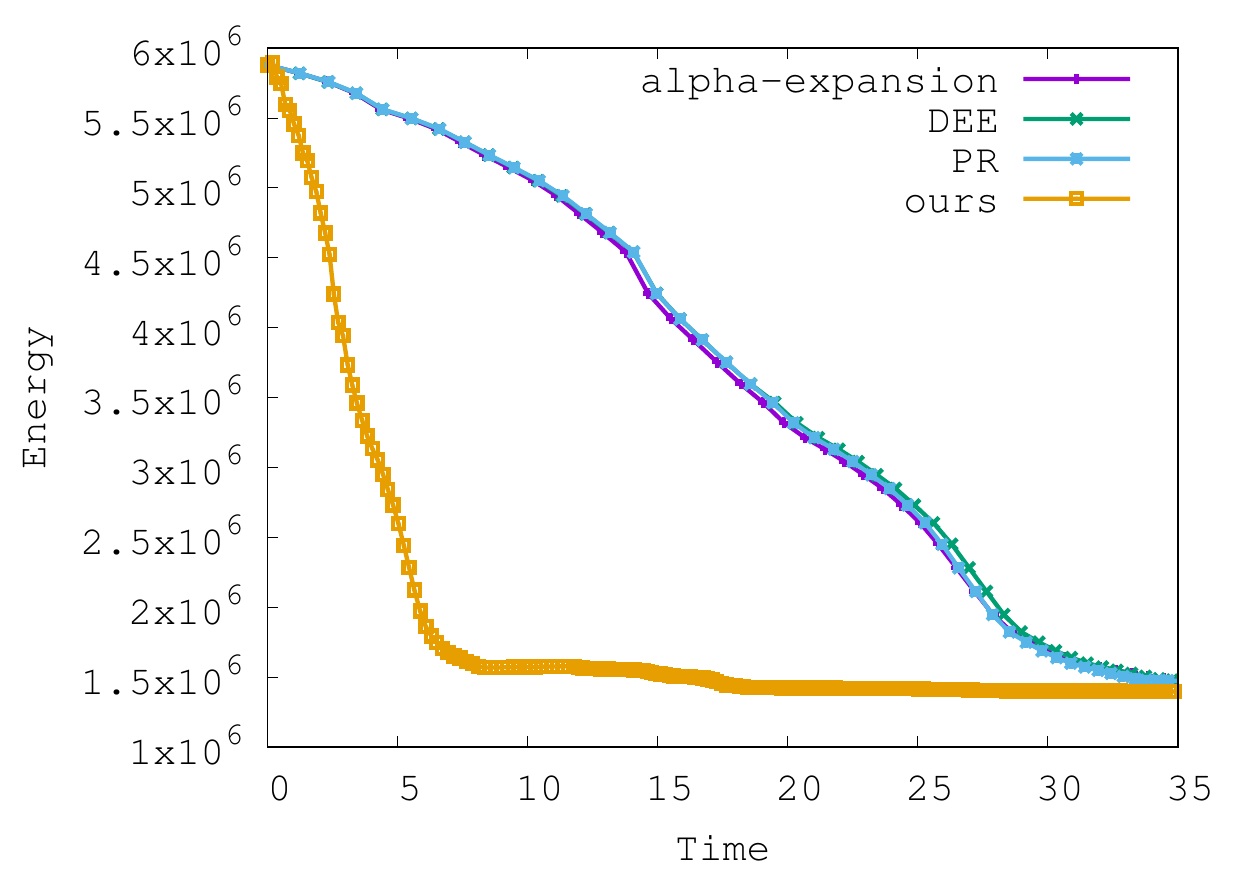}
  \caption{Energy vs time curve on instance \texttt{Ted} in Stereo dataset}\label{fig:speedcurve}
\end{figure}

\subsection{Results on benchmarks}\label{sec:benchmarks}

We summarize our experimental results in Table~\ref{tab:experiments}. Our
primary goal is to speedup MRF inference on hard problems, and there is
evidence that our benchmarks are challenging. The state-of-the-art IRI
method, which delivers impressive performance on the easier problems in our
benchmarks, struggles with the harder problems\footnote{In the table, time out
means we do not obtain results after running for 10x the overall running time
that expansion moves take.} while MPQBO runs out of memory. The only
source code for Kovtun we could find is restricted to the
Potts model.

Our approach achieved a significant improvement, making expansion moves 2x
to 12x faster on various datasets. Our pre-processing method
beats its natural competitor DEE by around 2x, and outperforms all the
baseline methods. Figure~\ref{fig:speedcurve} shows a typical energy vs. time
curve. We can see our approach drives the energy
curve down much faster than the other methods.

The key factor for the speedup is the percentage of labeled variables. The
values of these variables are fixed during the pre-processing step, 
resulting in a smaller problem for max-flow/QPBO to solve. 
Table~\ref{tab:experiments} shows our approach labels significantly
more variables than DEE and PR, especially on the
% object-seg,
inpainting and denoising-sq datasets.
Kovtun, MQPBO and IRI have very expensive overhead as the
pre-processing step. While it is impressive that IRI labels almost every
variable on the easy dataset, it is still 2x-6x slower
than our proposed method. Furthermore, Kovtun, MQPBO and IRI do not perform
well on our challenging datasets. When the size of the label set is large
(which is common in many vision problems such as inpainting, denoising or
optical flow), even IRI only proves a few variables to be persistent after
spending 3x-70x as much time as our method. This demonstrates the
advantage of performing pre-processing on the binary subproblem, which is
consistent with the observation in \cite{WZ:CVPR16}. 

Our method also performs well in terms of energy, especially on the hard
benchmarks. Because we can label some variables incorrectly during
pre-processing, there is a risk of producing a larger energy. However, the
experimental results are reassuring: on the hard problems we actually produce
slightly lower energy, while on the easier problems we can produce slightly
higher energy.

While it is somewhat counter-intuitive, occasionally labeling variables
incorrectly can plausibly lead to a better overall energy by getting out of a
local minimum. Expansion moves can be viewed as a local search algorithm
although its search space has an exponential size
\cite{Orlin:VLSN02}. Therefore, a random walk going uphill occasionally may
help us escape from the local minimizer, as in the Metropolis algorithm
\cite{Metropolis53} or simulated annealing \cite{Kirkpatrick83}. At one
iteration of the expansion move algorithm, our method may label some variables
incorrectly and solve the binary subproblem suboptimally (i.e., our
pre-processing may cause the energy to increase during expansion move
framework). It is plausible that this suboptimal move for the binary subproblem may
also help us escape from the local minimizer. To verify this hypothesis, we
experimented with a variant of our method where we reject an expansion move if
it makes the energy worse. In experiments, this change led to a worse final energy. 
This suggests that allowing suboptimal moves is beneficial.

We believe that our method achieves competitive energy due to the very high precision, 
shown in Table~\ref{tab:precision}.
It demonstrates that our discriminative rule described in
Eq.~\ref{eq:decisionrule} is effective and powerful despite being
simple and
intuitive. In general, by compromising precision a little bit, we can
significantly boost the recall value, as illustrated in Table~\ref{tab:kappa}.

In summary, our proposed method achieves a high quality trade-off between
running time and energy among all the methods, particularly on challenging
datasets. It runs significantly faster than its competitors and achieves an
energy that is similar and sometimes even lower.

\subsection{Experiments with parameters and bounds}\label{sec:epsilon}

In Section~\ref{sec:benchmarks}, we set the parameter $\epsilon = \infty$, and
investigated how our algorithm performed without the worst case bound. We
have demonstrated the proposed discriminative rule Eq.~\ref{eq:decisionrule}
itself is empirically effective. All the post-running per-instance bounds we
proved in Section~\ref{sec:perinstancebound} are still sound, although in this
variant of our method there is no worst case theoretical guarantee.

However, if we combine Eq.~\ref{eq:decisionrule} and $\bar \delta_i \le
\epsilon$ as our decision rule, as described in Section~\ref{sec:worstcasebound},
we will have the worst case
bounds. We also conducted experiments with different $\epsilon$'s. Our
observation is that when $\kappa \ge 0.8$, adding the rule $\bar \delta_i \le
\epsilon$ has minimal effects on the speedup and energy we reported in
Table~\ref{tab:experiments}, since our precision is already very high as shown
in Table~\ref{tab:kappa}. However, it gives us a worst case theoretical
guarantee. When $\kappa \le 0.6$, we can observe a noticeable improvement on
precision and energy change when we decrease the $\epsilon$ value with other
parameters fixed. As a special case, we have a sound condition again when
$\epsilon = 0$. In general, decreasing $\epsilon$ increases the running time,
but the tradeoffs involved are not obvious, and we defer details to the
supplementary material.

%\section*{Acknowledgments}
\vspace{1em}
\noindent\textbf{Acknowledgments:} This research was supported by NSF grants IIS-1447473 and IIS-1161860 and by a
Google Faculty Research Award.
\newpage
\onecolumn

\section{Outline of supplementary material}\label{sec:outline}
We will give a detailed running time analysis of our proposed algorithm in Section~\ref{sec:runningtime}. Then we will give the proof to Lemma~4 and Theorem~10 in Section~\ref{sec:monotonecriterion} and Section~\ref{sec:multiplicativeboundgraphcut} respectively. Generalization of the efficient discriminative criterion check subroutine will be described in Section~\ref{sec:generalization}. More implementation details will be given in Section~\ref{sec:moreimplementation}. Finally, we will provide more experimental data in Section~\ref{sec:moreexperiments}, including visualization results, experimental results on a typical parameter setup, more investigation on parameters sensitivity, the role of worst case bound in practice and preliminary results on multilabel MRFs.

\section{Running time analysis}\label{sec:runningtime}
\RestyleAlgo{boxruled}
\LinesNumbered
\begin{algorithm} % [H]
 \SetAlgoNoLine
 \KwIn{Energy function $E(x)$}
 $\hat x \gets \emptyset$; \quad $S \gets \emptyset$\;
 \For{$t \gets 1$ \textbf{to} $\tau$}{
 	\For{$i \in V \backslash S, \ell \in \calL_i$}{
		%\If{$|\calL_i| = 1$}{\textbf{continue}\;}
		Compute $LB \le\sum_{z_\Ni \in \hat \calL_\Ni(x_i = \ell)} q(z_\Ni)$\;
		\If{$LB \ge \kappa$}{
			$\hat x \gets \hat x \oplus \{x_i = \ell\}$\;
			$\calL_i \gets \{\ell\}$; \quad $S \gets S \cup \{i\}$\;
		}
 	}
 }
 With $\hat x_S$ fixed, use one MRF inference algorithm to solve the remaining variables, get $\hat x_{V \backslash S}$\;
 \Return $\hat x = \hat x_S \oplus \hat x_{V \backslash S}$\;
 \caption{MRF inference with pre-processing}\label{alg:construction}
\end{algorithm}

The pseudo-code of our proposed algorithm is listed in Algorithm~\ref{alg:construction}. It's the same pseudo-code we have in the main paper.

We will give a asymptotic analysis on the running time of our pre-processing algorithm here. Assuming we have an oracle to give us data term $\theta_i(x_i)$ and prior term value $\theta_{ij}(x_i, x_j)$ in $\calO(1)$ time. Let $N = |V|, M = |E|$ and $L = \max_{i} |\calL_i|$ to be the number of variables, edges and maximum possible labels, $d = \max_i |\Ni|$ is the maximum degree of the graph. For a typical vision problem, we usually have a sparse graph like grid, meaning $\calO(N) = \calO(M)$ and $d$ is also usually a small constant like 4 or 8. 

%Estimating the probability $q$ on line 2 will take $\calO(NL)$ time if we apply the uniform distribution or just initialize it from the unary terms $q_i(x_i) = e^{-\theta_i(x_i)ß}$. If we apply the max-product belief propagation, it will takes $\calO(ML^2)$ time (we will just run it for a constant number of iterations).

Computation time of the for loop from line 2 to 10 needs some thinking. $\tau$ is usually a small constant, so we can omit it in the asymptotic analysis. For the given $x_i = \ell$, a naive implementation of brute force algorithm to compute $\sum_{z_\Ni \in \hat \calL_\Ni} q(z_\Ni)$ needs to enumerate all the possible neighboring configurations $z_\Ni$, and it takes $\calO(dL)$ to compute $\min_{y_i \ne x_i}\energychangedefi$, so it takes $\calO(d L^{d+1})$ time. Therefore, the overall running time is $\calO(dNL^{d+2})$ for brute force so it's still feasible when both d and L are small constant.

When we use the approximated way to compute the lower bound using Lemma~5 in the main paper, we need an faster way to compute Eq.~9. 
We can pre-compute all the terms we may used here in $\calO(NL + EL^2)$ time globally and then query it in $\calO(d)$ time without solving the min operator each time. 
Then it takes $\calO(d^2L)$ time to compute $\calA_j$, $\calO(dL)$ time to compute $Q_i$ and $\calO(d)$ to compute the sum each iteration. 
Also note that once we fix a variable, it also takes $\calO(L+dL^2)$ to update our pre-computations result.
But each variable will only be fixed at most once during the pre-processing, so the amortized running time to update the pre-computations result is $\calO(NL + EL^2)$.
So in sum, we have the overall running time $\calO(d^2NL^2 + EL^2)$ for approximated calculation. 

\setcounter{thm}{3}
\section{Proof of Lemma~\ref{lem:monotonecriterion}}\label{sec:monotonecriterion}
\begin{lem}\label{lem:monotonecriterion}
For the same set of decision problems for persistency, we will never increase
the number of false positives by increasing $\kappa$.  
\end{lem}
\begin{proof}
This one is trivial. Consider any non-persistent $x_S$, it will be a false positive with parameter $\kappa_2$ if and only if it meets our discriminative criterion, i.e., $\sum_{z_\NS \in \hat \calL_{\calN(S)}} q(z_\NS) \ge \kappa_2$. Now for the algorithm using parameter $\kappa_2 > \kappa_1$, our discriminative criterion still holds, hence it's still a false positive for our algorithm with parameter $\kappa_1$.
\end{proof}

\setcounter{thm}{9}
\section{Proof of Theorem~\ref{thm:multiplicativeboundgraphcut}}\label{sec:multiplicativeboundgraphcut}
\begin{thm}\label{thm:multiplicativeboundgraphcut}
Suppose we use expansion movess as the inference algorithm, with the
$\beta$-multiplicative bound, then we will have $E(\hat x) \le \beta\cdot
E(x^*) + |S|\epsilon$. 
\end{thm}
\begin{proof}
Following the proof of the multiplicative bound of expansion moves algorithm~\cite{BVZ:PAMI01} (Theorem~6.1), we will see actually the multiplicative factor $\beta$ will not be applied to unary terms. In other words, $E'(\hat x) = \sum_i \theta'_i(\hat x_i) + \sum_{ij}\theta'_{ij}(\hat x_i, \hat x_j) \le \sum_i \theta'_i(x^*_i) + \beta\sum_{ij}\theta'_{ij}(x^*_i, x^*_j) \le \beta E'(x^*)$.

Note that in our algorithm, the energy function $E'(x)$ of expansion moves is induced by fixing $\hat x_S$ in $E(x)$, all the pairwise terms $\theta_{ij}$ crossing $S$ and $V \backslash S$ could be viewed as the unary terms in $E'(x)$ since one variable will be fixed. Therefore, we will have following.

\begin{equation}
\begin{aligned}
 &E(\hat x_S \oplus \hat x_{V \backslash S}) \\
= & \sum_{i \in S} \theta_i(\hat x_i) + \!\!\!\!\sum_{i,j \in S, (i, j) \in E} \!\!\!\! \theta_{ij}(\hat x_i, \hat x_j) + \!\!\!\! \sum_{i \in S, j \in V \backslash S, (i, j) \in E}\!\!\!\! \theta_{ij}(\hat x_i, \hat x_j) + \sum_{i \in V \backslash S} \theta_i(\hat x_i) + \!\!\!\!\sum_{i,j \in V \backslash S, (i, j) \in E}\!\!\!\! \theta_{ij}(\hat x_i, \hat x_j)\\
\le & \sum_{i \in S} \theta_i(\hat x_i) + \!\!\!\!\sum_{i,j \in S, (i, j) \in E}\!\!\!\! \theta_{ij}(\hat x_i, \hat x_j) + \!\!\!\!\sum_{i \in S, j \in V \backslash S, (i, j) \in E}\!\!\!\! \theta_{ij}(\hat x_i, x^*_j) + \sum_{i \in V \backslash S} \theta_i(x^*_i) + \beta \!\!\!\!\sum_{i,j \in V \backslash S, (i, j) \in E}\!\!\!\! \theta_{ij}(x^*_i, x^*_j)\\
 \le & \sum_{i \in S} \theta_i(x^*_i) + \!\!\!\!\sum_{i,j \in S, (i, j) \in E}\!\!\!\! \theta_{ij}(x^*_i, x^*_j) + \!\!\!\!\sum_{i \in S, j \in V \backslash S, (i, j) \in E}\!\!\!\! \theta_{ij}(x^*_i, x^*_j) + \sum_{i \in V \backslash S} \theta_i(x^*_i) + \beta \!\!\!\!\sum_{i,j \in V \backslash S, (i, j) \in E}\!\!\!\! \theta_{ij}(x^*_i, x^*_j) + |S|\epsilon\\
  \le &\beta\cdot E(x^*) + |S|\epsilon.
\end{aligned}
\end{equation}
\end{proof}

\section{Generalization of the efficient check of discriminative criterion}\label{sec:generalization}
When we want to decide if the given partial labeling $x_S$ is persistent or not, we can follow exactly the same idea presented in Section~3.3 of the main paper to compute the lower bound of $\sum_{z_\NS \in \hat \calL_{\calN(S)}} q(z_\NS)$. The only big difference is that we need a subroutine to efficiently check $\min_{y_S \ne x_S} \energychangedefNS > 0$ for $z_\NS \in \calL_\NS$ with $z_j = \ell$. Persistency relaxation (PR)~\cite{WZ:CVPR16} generalizes dead end elimination (DEE)~\cite{Desmet:Nature92} from checking persistency of a single variable $x_i$ to an independent local minimum (ILM) partial labeling $x_S$. The subproblem in PR is to decide if $\min_{y_S \ne x_S} \energychangedefNS > 0$ for $z_\NS \in \calL_\NS$, without the additional constraint that $z_j = \ell$, and they proposed a bunch of sufficient conditions to efficiently check it. Actually, it's trivial to enforce the additional constraint $z_j = \ell$ in those approaches. We just need to remove $z_j$ from the free variables and force it takes value $\ell$ in the subroutine proposed in PR. Note that those subroutines are sound so we can still apply Lemma~5 to partial labeling $x_S$ and get the lower bound of $\sum_{z_\NS \in \hat \calL_{\calN(S)}} q(z_\NS)$. Once we have our discriminative criterion as the decision subroutine, we can follow the construction algorithm in PR (Algorithm~2) as the generalization of our proposed construction algorithm in the main paper.

\section{More implementation details}\label{sec:moreimplementation}
Since we applied the proposed method to each induced binary subproblem in the expansion moves algorithm, we only check persistency for $x_i = 0$ (i.e., do not take move in the binary case) after the first epoch of running expansion moves algorithm in order to get the maximum speedup. We observed that after the first epoch, most of the variables won't change its value, hence the extra benefit from checking persistent for $x_i = 1$ is very marginal.

\ignore{
Another implementation detail for generating Table~5 of the main paper. Note that by the nature of the proposed algorithm, it is possible have our discriminative criterion for both $x_i = 0$ and $x_i = 1$ to be true when $\kappa$ is small, since we allow false positives. However, in Algorithm~1 of the main paper, we just fix its value to be 0. Therefore, in order to generate the correct precision, recall values for small $\kappa$ values, we will run the algorithm for two passes. One pass only checks $x_i = 0$ and the other pass only checks $x_i = 1$. Then we will aggregate the results. Note that it won't affect other experiments in the main paper since their $\kappa$ values are high enough so that our discriminative criterion claims $x_i = 0$ to be persistent and $x_i = 1$ to be persistent exclusively.
}

\section{Additional experimental results}\label{sec:moreexperiments}
\subsection{Visualization results}\label{sec:visualization}
We presented the visualization results on the stereo task in Fig.~\ref{fig:stereovis}. We can see there is no significant visual difference between the expansion moves results and our results, even in the case that our method has slightly higher energy. Therefore, it's appealing to apply our method in practice, since it has almost the same visual quality but makes the inference much faster. When we set up a limited time budget in real applications, see the second column of Fig.~\ref{fig:stereovis}, our approach can generate much better visual result than regular expansion moves algorithm without pre-processing. In this case, regular expansion moves even doesn't finish its first epoch and has a very poor disparity map.

\begin{figure}
        \centering
        \begin{subfigure}[b]{0.33 \linewidth}
            \centering
            \includegraphics[width=\textwidth]{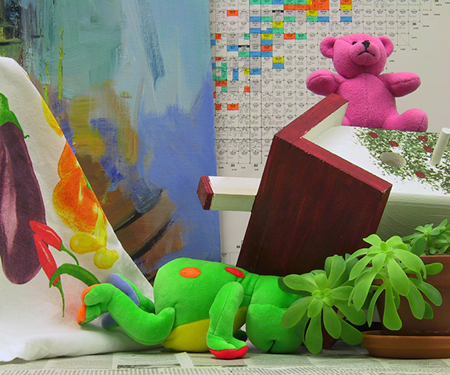}
            \caption{Reference image}\label{fig:stereo-img}
        \end{subfigure}
        \begin{subfigure}[b]{0.33 \linewidth}  
            \centering 
            \includegraphics[width=\textwidth]{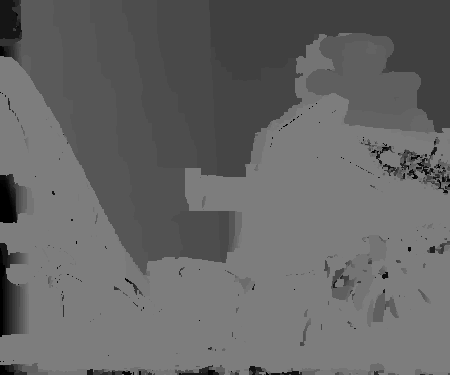}
            \caption{$\alpha$-expansion at 25s ($E = 2731940.0000$)}\label{fig:stereo-alpha-25s}
        \end{subfigure}
        \begin{subfigure}[b]{0.33 \linewidth}   
            \centering 
            \includegraphics[width=\textwidth]{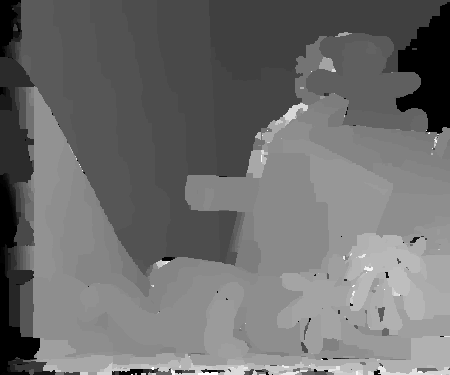}
            \caption{$\alpha$-expansion ($E=1343617.0000$)}\label{fig:stereo-alpha}
        \end{subfigure}
        \begin{subfigure}[b]{0.33 \linewidth}   
            \centering 
            \includegraphics[width=\textwidth]{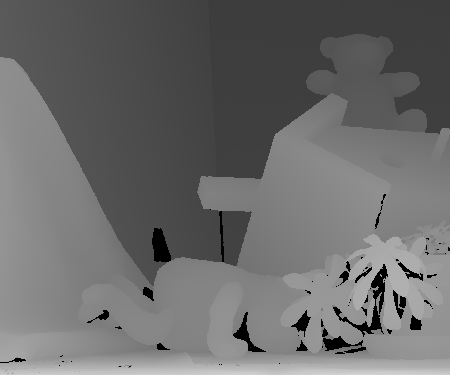}
            \caption{Ground truth}\label{fig:stereo-gt}
        \end{subfigure}
        \begin{subfigure}[b]{0.33 \linewidth}   
            \centering 
            \includegraphics[width=\textwidth]{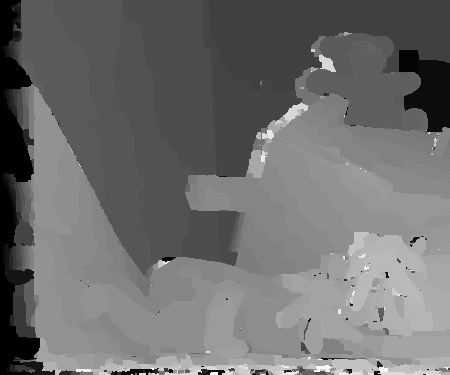}
            \caption{Ours at 25s ($E = 1418450.0000$)}\label{fig:stereo-deeprob-25s}
        \end{subfigure}
        \begin{subfigure}[b]{0.33 \linewidth}   
            \centering 
            \includegraphics[width=\textwidth]{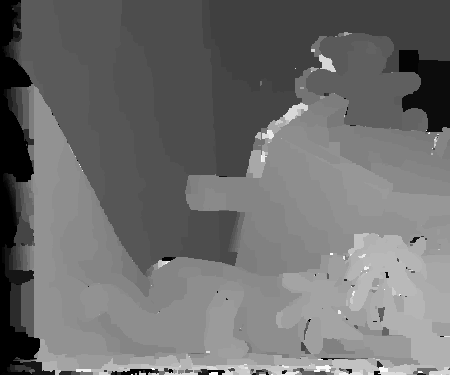}
            \caption{Ours ($E = 1391959.0000$)}\label{fig:stereo-deeprob}
        \end{subfigure}
        \caption{Stereo instance \texttt{Teddy}}\label{fig:stereovis}
    \end{figure}
    
\subsection{Experimental results with a typical parameter setup}\label{sec:typicalparameter}
Our experiments suggest that the proposed method can achieve good performance with
the parameters in a wide range. We report the experimental results in Table~\ref{tab:noloo} 
with the following fixed parameters to avoid the expense of cross-validation: $\kappa = 0.8,
\tau = 3$, using the uniform distribution for $q(x)$ and checking with
our efficient subroutine described in Section~3.3 of the main paper. 

Even though this is a fairly conservative assumption (we use the exact same parameters for very different energy
functions), we still obtain good results.  We acheive a 2x-12x speedup
on different datasets with the energy increasing $~0.1\%$ on the worst case.  In
addition, we still get lower energy on 4 of the 5 challenging dataset. 

We also listed the performance of our method with the parameters selected with the 
leave-one-out cross validation procedure as a reference 
(shown in Table~2 in the main paper).
We see that the performance of our method is very similar no matter whether we use fixed parameters or use cross validation to choose the parameters.
The key observation of the main paper still holds even with this fixed typical parameter
setup, i.e., our method achieves significant speedup against baseline methods with very minor
compromise on the accuracy of the partial optical labelings (usually lose $<0.5\%$ precision).
We also achieve comparable or smaller energy even though we compromise the accuracy of the 
partial optical labelings in the pre-processing step.

Therefore, these experiments demonstrate that it's sufficient to use the typical parameter
setup of our method in practice. We can achieve very good performance without using the expensive
cross validation parameter selection procedure.

\begin{table*}[!t]
\centering
\caption{Performance of our method on a typical parameter setup}\label{tab:noloo}
\small
\begin{tabular}{|c|c|c|c|c|c|c|c|}
\hline
\multicolumn{8}{|c|}{Typical parameter setup (w/o cross validation)}\\
\hline
Dataset & \bf Stereo & \bf Inpainting & \bf Denoising-sq & \bf Denoising-ts & \bf Optical Flow & \bf Color-seg-n4 & \bf Color-seg-n8 \\
\hline
Speedup & 2.14x & 2.10x & 11.71x & 10.61x & 8.92x & 9.31x & 8.45x \\
\hline
Energy Change & -0.04\% & -0.53\% & -0.03\% & -0.09\%& +0.11\% & +0.01\% & +0.05\% \\
\hline
Labeled Vars & 56.77\%& 47.06\%& 97.39\%& 96.64\%& 93.74\% & 90.80\% & 90.43\% \\
\hline
Precision & 99.71\% & 99.88\% & 99.95\% & 99.95\% & 99.50\% & 99.50\% & 99.76\% \\
\hline
\multicolumn{8}{|c|}{Leave one out parameter selection (w/ cross validation)}\\
\hline
Dataset & \bf Stereo & \bf Inpainting & \bf Denoising-sq & \bf Denoising-ts & \bf Optical Flow & \bf Color-seg-n4 & \bf Color-seg-n8 \\
\hline
Speedup & 1.78x & 3.40x & 11.83x & 11.91x & 4.69x & 7.02x & 8.33x \\
\hline
Energy Change & -0.06\% & -1.71\% & -0.02\% & 0.00\%& -0.04\% & 0.00\% & +0.04\% \\
\hline
Labeled Vars & 44.76\%& 74.29\%& 97.91\%& 98.32\%& 77.25\% & 85.74\% & 90.39\% \\
\hline
Precision & 99.74\% & 96.16\% & 99.95\% & 99.79\% & 99.88\% & 99.79\% & 99.77\% \\
\hline
\end{tabular}
\vspace{-1em}
\end{table*}%

\subsection{Investigation on parameter sensitivity}\label{parameters}
We claimed in the main paper that the parameters chosen by the leave-one-out procedure are very similar for the same dataset. We summarized the parameters chosen by cross validation in Table~\ref{tab:parameters}. The exception column shows that, out of the 7 datasets we tested, the leave-one-out procedure only results in 5 cases where the parameters are different from the majority of the dataset. We also observed that the exception instance achieves good performance when applied to the majority parameter setup of the whole dataset. Therefore, we conclude the best parameter suit for one dataset is quite stable, and the parameters chosen from a set of energy can still be applicable to other energy functions derived from the same vision task.

In addition, we also observed that the proposed method achieves good performance across 
all the datasets we tested when the parameters are chosen from a wide range, including the typical
parameter setup we reported in Section~\ref{sec:typicalparameter}. Therefore, the proposed method
is robust to its parameters.

\ignore{
We have already studied the role of $\kappa$ in the main paper. So we will investigate how does the choice of $q$ affect our algorithm. One interesting finding from Table~\ref{tab:parameters} is if we really want to have a good speedup, we just need to choose $q$ as uniform distribution or the distribution derived from unary terms. We don't need the fancy max-product message passing estimation at all, which has expensive running time overhead. It also suggests our algorithm is very robust to the imperfect estimation of the MAP local marginal. However, we still want to learn how does the choice of $q$ may affect our algorithm from other aspects. Therefore, we varied $q$ to be the MAP local marginal estimation from running max-product LBP for $\{0,1,3,5\}$ iterations and summarized the results in Table~\ref{tab:lbp}. Note the running LBP for 0 iteration is just the special case to use the initialization from unary term as $q$ directly. We can see it's obvious that the more iterations we run LBP, the tighter our approximation $q$ is. It can help use get both higher precision and recall. However, considering the extra computational overhead of running LBP, we may not achieve the best speedup when we increase the number of iterations running LBP in practice.
}

\begin{table}[!t]
\begin{center}
\caption{Parameters chosen from the leave-one-out procedure}\label{tab:parameters}
\begin{tabular}{|c|c|c|c|c|c|}
\hline
  Dataset & $\kappa$ & Choice of $q$ & criterion check & Exception\\
\hline
\textbf{Stereo} & 0.8 & uniform & approximate & none\\
\hline
\textbf{Inpainting} & 0.7 & uniform & approximate & 1 instance with unary distribution\\
\hline
\textbf{Denoise-sq} & 0.8 & uniform & approximate & 1 instance with $\kappa = 0.9$\\
\hline
\textbf{Denoise-ts} & 0.7 & uniform & approximate & 1 instance with $\kappa = 0.8$\\
\hline
\textbf{Optical Flow} & 0.9 & unary & exact & 1 instance with $\kappa = 0.8$, approximate check\\
\hline
\textbf{Color-seg-n4} & 0.9 & unary & exact & 1 instance with $\kappa = 0.8$, uniform distribution, approximate check\\
\hline
\textbf{Color-seg-n8} & 0.9 & unary & exact & 1 instance with $\kappa = 0.8$\\
\hline
\end{tabular}
\end{center}
\end{table}

\ignore{
\begin{table}[!t]
\begin{center}
\caption{Experimental results with different iterations of LBP on Stereo dataset}\label{tab:lbp}
\begin{tabular}{|c|c|c|c|c|}
\hline
  \# iterations of LBP & Speedup & Precision & Recall & F1-score \\
\hline
0 & 3.52x  & 0.7657 & 0.6691 & 0.7139\\
\hline
1 & 3.81x & 0.9550 & 0.9326 & 0.9437\\
\hline
3 & 3.03x & 0.9676 & 0.9444 & 0.9559\\
\hline
5 & 2.41x & 0.9734 & 0.9525 & 0.9629\\
\hline
\end{tabular}
\end{center}
\end{table}
}

\subsection{Experimental results for worst case bounds}\label{sec:worstcaseboundexp}
In the main paper, we set $\epsilon = \infty$ to investigate how our algorithm performs without the worst case bound. We demonstrated that our algorithm can achieve very good performance in practice without it. Now we will study the role of $\epsilon$ in practice.

We conducted experiments on Color-seg-n4 dataset as an example. The experimental results are
summarized in Table~\ref{tab:epsilon}.

We firstly applied the typical parameter setup we used in Section~\ref{sec:typicalparameter}. 
The results are reported on the left part of Table~\ref{tab:epsilon}. Note that
$\epsilon = 0$ is the special case where our method only uses the sound condition to check the
partial optimal labeling, hence the proposed algorithm degenerates to the DEE algorithm. Therefore,
in this special case, we have a $100\%$ precision and label around $38\%$ variables, hence we get a 
moderate speedup without affecting the energy. We also know that $\epsilon = \infty$ is another special case
where we don't try to bound the worst case. These results is reported in the main paper and
Table~\ref{tab:noloo}. We already know that the fixed parameters we choose here are reasonable, so even in this 
extreme case, we still get good performance without the theoretical
guarantee. As $\epsilon$ decreases, we know that the criterion used becomes more strict. Therefore we will have
higher precision and less labeled variables. Due to that, 
we label fewer variables, and the speedup we acheive decreases. 
In this setup, since we always maintain the precision value at a extremely
high level, $\epsilon$'s impact on energy change is not that obvious.

To test this, we conducted the experiments under another set of purposely bad parameters, i.e., changed $\kappa = 0.6$. We summarized our results on the right part of Table~\ref{tab:epsilon}.
We see that with $\kappa = 0.6$ and large $\epsilon$ value (e.g., $\epsilon = 10$), our criterion is
loose enough to hurt the precision and result in $8\%$ higher energy than before.
In our experiments, we observed that as $\epsilon$ decreases from 10 down to 0, 
the precision increases
dramatically and the energy increment becomes smaller. Therefore, $\epsilon$ values not only
give us the theoretical worst case guarantee, but also make real impact in practice (make the criterion
we used close to the sound condition and makes the energy smaller).

\begin{table}[!t]
\begin{center}
\caption{Experimental results with different $\epsilon$ on Color-seg-n4 dataset}\label{tab:epsilon}
\begin{tabular}{|c||c|c|c|c||c|c|c|c|}
\hline
& \multicolumn{4}{|c||}{$\kappa = 0.8$} & \multicolumn{4}{|c|}{$\kappa = 0.6$}\\
\hline
  $\epsilon$ & Speedup & Energy Change & Labeled Vars & Precision & Speedup & Energy Change & Labeled Vars & Precision\\
\hline
0 & 4.16x & 0.00\% & 37.82\% & 100.00\% & 4.16x & 0.00\% & 37.82\% & 100.00\% \\
\hline
0.01 & 4.31 & 0.00\% & 67.73\% & 99.99\% & 4.46x & 0.00\% & 68.25\% & 99.97\%\\
\hline
0.1 & 6.05x & 0.00\% & 72.07\% & 99.93\% & 6.47x & +0.01\% & 73.69\% & 99.70\%\\
\hline
0.2 & 6.68x & 0.00\% & 74.67\% & 99.86\% & 7.93x & +0.20\% & 78.32\% & 99.34\%\\
\hline
0.3 & 6.97x & 0.00\% & 76.32\% & 99.81\% & 8.38x & +0.34\% & 81.43\% & 99.12\%\\
\hline
0.4 & 7.07x & 0.00\% & 77.86\% & 99.80\% & 9.62x & +1.29\% & 85.91\% & 98.68\%\\
\hline
0.5 & 7.59x & 0.00\% & 81.16\% & 99.74\% & 11.73x & +3.01\% & 88.41\% & 97.59\%\\
\hline
%0.6 & 7.79x & +0.01\% & 82.49\% & 99.67\% & 12.06x & +6.88\% & 89.29\% & 96.54\%\\
%\hline
%0.7 & 7.84x & +0.01\% & 82.67\% & 99.67\% & 12.14x & +6.88\% & 89.82\% & 96.54\%\\
%\hline
%0.8 & 7.83x & +0.01\% & 84.87\% & 99.69\% & 12.29x & +6.88\% & 95.67\% & 96.54\%\\
%\hline
%0.9 & 7.87x & +0.01\% & 86.43\% & 99.69\% & 12.23x & +6.88\% & 96.01\% & 96.53\%\\
%\hline
1.0 & 7.92x & +0.01\% & 88.48\% & 99.69\% & 12.38x & +6.88\% & 96.25\% & 96.51\%\\
\hline
10.0 & 8.12x & +0.01\% & 90.80\% & 99.50\% & 15.02x & +7.83\% & 98.52\% & 94.77\%\\
\hline
\end{tabular}
\end{center}
\end{table}

\subsection{Comparison to other MRF inference algorithm}
The main focus of this paper is to demonstrate that the proposed decision criterion is efficient and
effective in finding a partial optimal labeling of MRFs. We achieve a very good tradeoff between
the running time and the final energy by employing our proposed method as the pre-processing
for the expansion moves algorithm. 

Demonstrating that expansion moves is a state-of-the-art MRF inference algorithm is not the main
goal of this paper. The comparison among different inference algorithms are provided in survey
papers~\cite{OpenGM:IJCV15,SZSVKATR:PAMI08}. However, for the completeness of the paper,
we still perform the experiments
comparing against other widely used MRF inference algorithms besides
expansion movess, including loopy belief propagation
(LBP)~\cite{Murphy12,WF:TIT01}, dual decomposition
(DD)~\cite{Kappes12:Bundle}, TRWS~\cite{Kolmogorov:TRWS06} and
MPLP~\cite{globerson2008fixing, sontag2012efficiently, sontag2008tightening}.
 
The experimental results are reported in Table~\ref{tab:experiments}.
We set the time budget for the baseline methods as the 10x of the running time used by
expansion moves.
In our experiments, expansion movess are usually significantly faster than other methods, 
and results in comparable or even better energy. This observation is consistent with
the survey papers~\cite{OpenGM:IJCV15,SZSVKATR:PAMI08}. We can see that LBP, DD,
and MPLP usually will get higher energy compared to expansion moves even with 10x of time
budget. TRWS is promising since it can provide (slightly) lower energy than expansion moves,
although it's much slower. On the datasets we tested, TRWS will spend 3-10x longer time to get 
energy comparable to our proposed method, through its final energy might be slightly smaller. 
Typical energy-time curves are presented in Fig.~\ref{fig:morespeedenergy}. We
can see that LBP, DD, TRWS are usually much slower than our method with comparable
converging energy.

\begin{table*}[!t]
\centering
\caption{Additional experimental results (TO: time out, MEM: out of memory)}
\label{tab:experiments}
\footnotesize
\begin{tabular}{|cc|c|c||c||c|c|c||c|c|c|c|}
\hline
&&Dataset & Measurement & Ours & DEE & PR & IRI & LBP & DD & TRWS & MPLP\\
\hline
\multirow{15}{*}{\rotatebox[origin=c]{90}{\bf Challenging Datasets}} &\hspace{-1em}\multirow{15}{*}{\rotatebox[origin=c]{90}{(non- Potts energy, large $|\calL|$)}} &\bf Stereo & Speedup & 1.78x & 1.06x & 1.13x & 0.51x & 0.17x & 0.10x & 0.10x & MEM\\
&& 12--20 labels & Energy Change & -0.06\% & 0.00\% & 0.00\% & -0.15\% & +86.55\% & +92.25\% & -0.63\% & MEM\\
&& Trunc. L1/L2 & Labeled Vars & 44.76\% & 10.07\% & 18.06\% & 56.45\% & - & - & - & MEM \\
\cline{3-12}
&& \bf Inpainting & Speedup & 3.40x & 1.28x & 1.32x & 0.12x & 0.10x & 0.10x & 0.10x & MEM\\
&& 256 labels & Energy Change & -1.71\% & 0.00\% & 0.00\% & 0.00\% & +25.94\% & +51.39\% & -9.71\% & MEM\\
&& Trunc. L2 & Labeled Vars & 74.29\% & 21.05\% & 23.75\% & 0.36\% & - & - & - & MEM \\
\cline{3-12}
&& \bf Denoising-sq & Speedup & 12.76x & 1.15x & 1.33x & 0.29x & 0.10x & 0.09x & 0.10x & MEM\\
&& 256 labels & Energy Change & -0.02\% & 0.00\% & 0.00\% & 0.00\% & -0.65\% & +17.14\% & -0.65\% & MEM\\
&& L2 & Labeled Vars & 97.93\% & 17.42\% & 33.71\% & 0.39\% & - & - & - & MEM \\
\cline{3-12}
&& \bf Denoising-ts & Speedup & 13.08x & 11.97x & 11.86x & 0.18x & 0.10x & 0.09x & 0.10x & MEM\\
&& 256 labels & Energy Change & 0.00\% & 0.00\% & 0.00\% & -0.03\% & -0.78\% & +13.29\% & -0.99\% & MEM\\
&& Trunc. L2 & Labeled Vars & 98.22\% & 95.54\% & 97.71\% & 5.85\% & - & - & - & MEM \\
\cline{3-12}
&& \bf Optical Flow & Speedup & 4.69x & 2.63 & 3.40x & TO & 0.10x & 0.09x & 0.10x & MEM\\
&& 225 labels & Energy Change & -0.04\% & 0.00\% & 0.00\% & TO & +9.63\% & +16.07\% & -0.58\% & MEM\\
&& L1 & Labeled Vars & 77.25\% & 54.34\% & 65.51\% & TO & - & - & - & MEM\\
\hline
\multirow{6}{*}{\rotatebox[origin=c]{90}{\bf Easy Datasets}} & \hspace{-1em}\multirow{6}{*}{\rotatebox[origin=c]{90}{(Potts, small $|\calL|$)}} & \bf Color-seg-n4 & Speedup & 7.02x & 4.55x & 6.34x & 3.67x & 0.14x & 0.10x & 0.36x & 0.10x\\
&&4--12 labels & Energy Change & 0.00\% & 0.00\% & 0.00\% & -0.12\% & +1.72\% & +3.17\% & -0.13\% & +0.25\%\\
&& Potts & Labeled Vars & 85.74\% & 65.38\% & 77.50\% & 98.44\% & - & - & - & - \\
\cline{3-12}
&& \bf Color-seg-n8 & Speedup & 8.33x & 5.61x & 6.37x & 1.45x & 0.10x & 0.10x & 0.12x & 0.10x\\
&&4--12 labels & Energy Change & +0.04\% & 0.00\% & 0.00\% & -0.10\% & +0.39\% & +4.49\% & -0.11\% & +0.22\%\\
&& Potts & Labeled Vars & 90.39\% & 71.62\% & 82.05\% & 99.35\% & - & - & - & - \\
\hline
\end{tabular}
\end{table*}%

    \begin{figure}
        \centering
        \begin{subfigure}[b]{0.49 \linewidth}
            \centering
            \includegraphics[width=\textwidth]{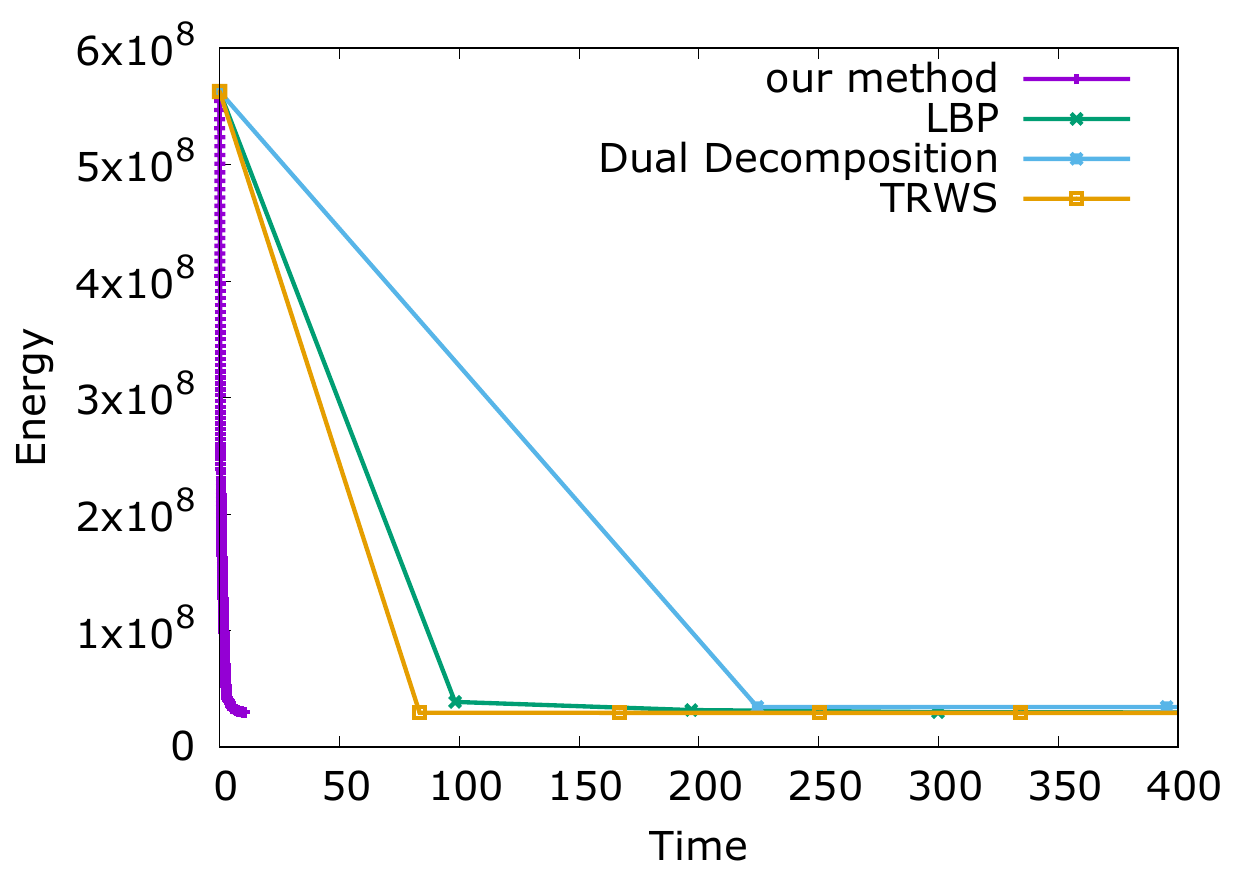}
            \caption{Instance \texttt{21077}}\label{fig:21077-curve}
        \end{subfigure}
        \begin{subfigure}[b]{0.49 \linewidth}  
            \centering 
            \includegraphics[width=\textwidth]{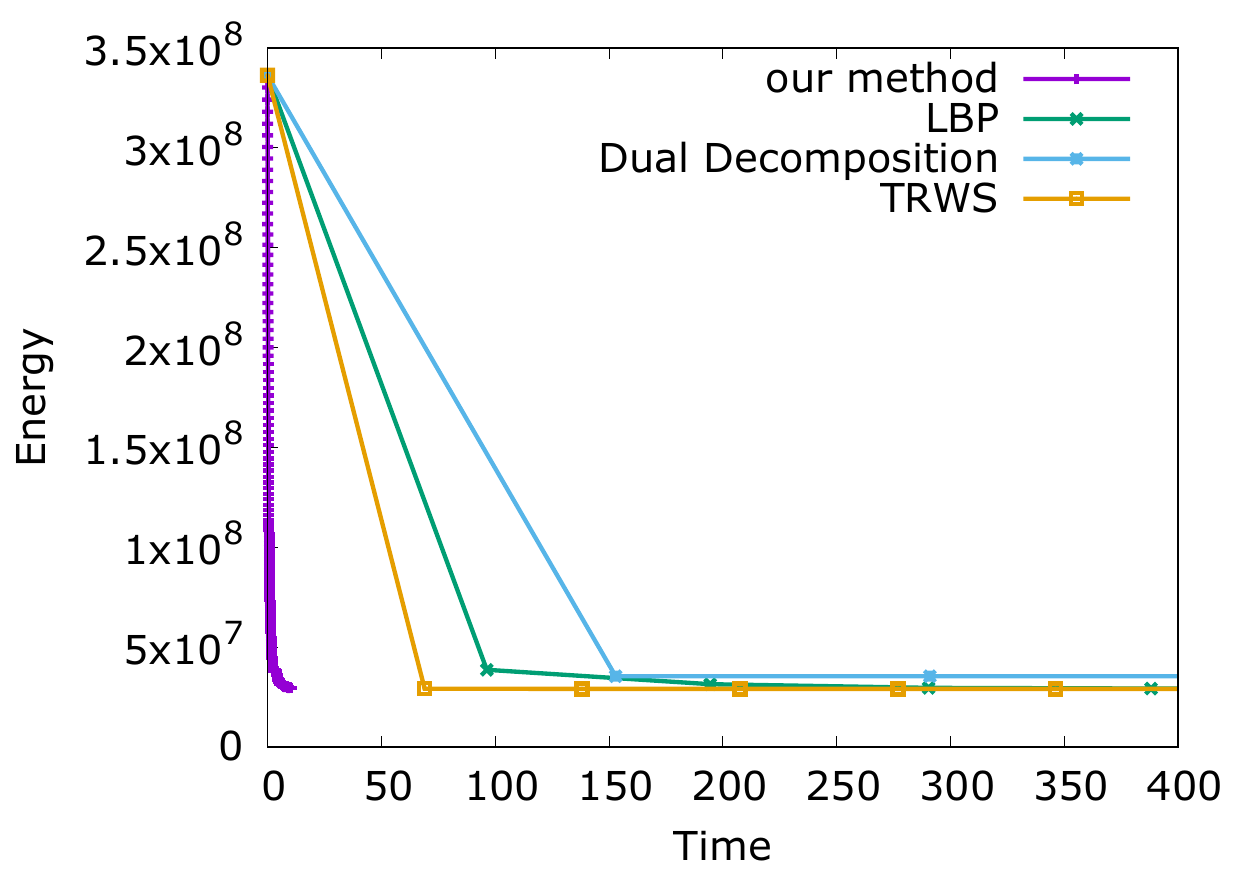}
            \caption{Instance \texttt{86000}}\label{fig:86000-curve}
        \end{subfigure}
        \caption{More speed-energy curve for denoise-sq dataset}\label{fig:morespeedenergy}
    \end{figure}

\subsection{Experimental results for multilabel MRFs}\label{sec:multilabel}
In the main paper, we mainly focused on applying the proposed pre-processing technique to each induced binary subproblem from the expansion moves algorithm. We can also apply the proposed pre-processing technique to the multilabel MRFs directly. We give preliminary results in Table~\ref{tab:multilabel}. Since the multilabel MRFs are NP-hard, it's very challenging to get the ground truth persistent labeling for each variable. Therefore, we don't report the precision/recall values. We just use the energy change as an indirect measurement to evaluate the quality of the persistent labeling we found. We also reported the percentage of labeled variables. Both metrics are computed in the per-dataset fashion. We set our distribution $q_i(x_i) = e^{-\theta_i(x_i)}$, which is only from the unary terms and used the fast approximation subroutine to check our discriminative criterion. Then we vary the $\kappa$ value in $\{0.6, 0.7, 0.8, 0.9\}$. We can see from Table~\ref{tab:multilabel} that the proposed method labels significantly more variables than the baseline method DEE, while increases the energy by a couple of percents. It's still the case that the proposed method can achieve a better tradeoff between the number of labeled variables and the energy we can get for the multilabel MRFs. In practice, it's more effective to apply our proposed method to each induced binary subproblem.

\begin{table}[!t]
\begin{center}
\caption{Preliminary experimental results on multilabel MRFs}\label{tab:multilabel}
\begin{tabular}{|c||c|c||c|c|}
\hline
  \multirow{2}{*}{Algorithm} & \multicolumn{2}{|c||}{Color-seg-n4} & \multicolumn{2}{|c|}{Color-seg-n8} \\
 & \multicolumn{1}{c}{\% of labeled variables} & Energy Change & \multicolumn{1}{c}{\% of labeled variables} & Energy Change \\
\hline
\Alpha{} & 0.00\% & 0.0000\% & 0.00\% & 0.0000\% \\
\hline
\DEE{} & 15.62\% & 0.0000\% & 19.67\% & 0.0000\%\\
\hline
\ProbDEE{} with $\kappa = 0.9$ & 25.55\% & +2.3589\% & 29.80\% & +0.1049\%\\
\hline
\ProbDEE{} with $\kappa = 0.8$ & 30.77\% & +3.3480\% & 30.62\% & +0.1161\%\\
\hline
\ProbDEE{} with $\kappa = 0.7$ & 34.01\% & +4.4572\% & 31.05\% & +0.1243\%\\
\hline
\ProbDEE{} with $\kappa = 0.6$ & 44.62\% & +5.7841\% & 31.27\% & +0.1297\%\\
\hline
\end{tabular}
\end{center}
\end{table}

\end{document}